\documentclass{article}
\usepackage[table,dvipsnames]{xcolor}         %
\usepackage{microtype}
\usepackage{graphicx}
\usepackage{subfigure}
\usepackage{booktabs} %

\usepackage{hyperref}

\usepackage[accepted]{icml2024}

\usepackage{amsmath}
\usepackage{amssymb}
\usepackage{mathtools}
\usepackage{amsthm}
\usepackage{titletoc}
\usepackage{etoc}

\usepackage[utf8]{inputenc} %
\usepackage[T1]{fontenc}    %
\usepackage{url}            %
\usepackage{booktabs}       %
\usepackage{amsfonts}       %
\usepackage{nicefrac}       %
\usepackage{microtype}      %

\usepackage[algo2e, lined,ruled,  linesnumbered, commentsnumbered, longend, noend]{algorithm2e}
\usepackage{graphicx}
\usepackage{array}
\usepackage{enumitem}
\usepackage{subcaption}
\usepackage{multirow}
\usepackage{graphics}
\usepackage{float}

\newcommand{\name}{\textsc{Aug-PE}\xspace{}}
\newcommand{\pename}{\textsc{PE}\xspace{}}

\newcommand{\ftgenerator}{\textsc{DP-FT-Generator}\xspace{}}
\newcommand{\ftdownstream}{\textsc{DP-FT-Downstream}\xspace{}}
\newcommand{\yelprating}{Rating}
\newcommand{\yelpcategory}{Category}
\newcommand{\openreviewrecom}{Rating}
\newcommand{\openreviewarea}{Area}
\newcommand{\berttiny}{$\text{BERT}_{\text{Tiny}}$}
\newcommand{\bertmini}{$\text{BERT}_{\text{Mini}}$}
\newcommand{\bertsmall}{$\text{BERT}_{\text{Small}}$}
\newcommand{\robertabase}{RoBERTa-base}

\newcommand{\llama}{LLaMA}
\newcommand{\llamatwosevenb}{LLaMA-2-7B}
\newcommand{\llamatwo}{LLaMA-2}
\newcommand{\gptfour}{GPT-4}
\newcommand{\gptthree}{GPT-3}
\newcommand{\gpttwo}{GPT-2}
\newcommand{\gpttwox}{GPT-2-series}
\newcommand{\gpttwom}{GPT-2-Medium}
\newcommand{\gpttwol}{GPT-2-Large}
\newcommand{\chatgpt}{GPT-3.5}
\newcommand{\gptthreepointfive}{GPT-3.5}
\newcommand{\claude}{Claude}
\newcommand{\bard}{Bard}

\usepackage{placeins}
\usepackage{tcolorbox}
\usepackage{listings}
\newtcolorbox{mybox}{
    colback=gray!20, %
    colframe=black, %
    arc=1mm, %
    boxrule=1pt, %
    left=1mm, %
    right=1mm, %
    top=1mm, %
    bottom=1mm %
}

\usepackage{tcolorbox}
\colorlet{shadecolor}{gray!20}
\definecolor{babyblueeyes}{rgb}{0.63, 0.79, 0.95}
\definecolor{babyblue}{rgb}{0.54, 0.81, 0.94}
\definecolor{bluegray}{rgb}{0.4, 0.6, 0.8}
\definecolor{cadmiumgreen}{rgb}{0.0, 0.42, 0.24}
\definecolor{camouflagegreen}{rgb}{0.47, 0.53, 0.42}
\definecolor{darkseagreen}{rgb}{0.56, 0.74, 0.56}
\definecolor{lightgray}{RGB}{239,240,241}

\usepackage{tikz}
\usetikzlibrary{fit,calc}
\newcommand*{\tikzmk}[1]{\tikz[remember picture,overlay,] \node (#1) {};\ignorespaces}
\newcommand{\boxit}[1]{\tikz[remember picture,overlay]
{\node[xshift=-0pt,yshift=-0pt,fill=#1,opacity=.15,fit={(A)($(B)+(0.9\linewidth,.8\baselineskip)$)}] {};}\ignorespaces}      

\colorlet{mypink}{red!30}
\colorlet{myblue}{cyan!50}
\colorlet{mygray}{gray!60}

\newcommand{\highlightbox}[1]{\colorbox[RGB]{239,240,241}{#1}}

\definecolor{c0}{cmyk}{1,0.3968,0,0.2588} 
\definecolor{c1}{cmyk}{0,0.6175,0.8848,0.1490} 
\definecolor{c2}{cmyk}{0.1127,0.6690,0,0.4431} 
\definecolor{c3}{cmyk}{0.3081,0,0.7209,0.3255} 
\newtcbox{\hlprimary}{on line,colback=c0!10,colframe=white,size=fbox,arc=3pt, box align=base,before upper=\strut, top=-2pt, bottom=-4pt, left=-1pt, right=-1pt, boxrule=0pt}
\newtcbox{\hlprimarytab}{on line, box align=base, colback=c0!10,colframe=white,size=fbox,arc=3pt, before upper=\strut, top=-2pt, bottom=-4pt, left=-2pt, right=-2pt, boxrule=0pt}
\newtcbox{\hlsecondary}{on line,colback=c1!10,colframe=white,size=fbox,arc=3pt, box align=base,before upper=\strut, top=-2pt, bottom=-4pt, left=-1pt, right=-1pt, boxrule=0pt}
\newtcbox{\hlsecondarytab}{on line, box align=base, colback=c1!20,colframe=white,size=fbox,arc=3pt, before upper=\strut, top=-2pt, bottom=-4pt, left=-2pt, right=-2pt, boxrule=0pt}
\newtcolorbox{hlmultiline}{on line,colback=decentgrey!75,colframe=white,size=fbox,arc=3pt, box align=base, top=0pt, bottom=2pt, boxrule=0pt, before=\adjustbox{valign=c}\bgroup, after=\egroup, before upper=\strut}

\newcolumntype{Y}{>{\centering\arraybackslash}X}
\newcolumntype{Z}{>{\raggedleft\arraybackslash}X}

\newcommand{\uashifted}{{\tiny$\uparrow$}}
\newcommand{\ua}[1]{{\scriptsize\hlsecondarytab{\uashifted{#1}}}}

\definecolor{c4}{cmyk}{0.6765,0.2017,0,0.0667} 
\definecolor{c5}{cmyk}{0,0.8765,0.7099,0.3647} 

\definecolor{darkgrey}{RGB}{149,149,149}
\definecolor{decentgrey}{RGB}{242,242,242}

\usepackage{amsthm,amsmath}
\theoremstyle{plain}
\newtheorem{theorem}{Theorem}

\usepackage[nameinlink,capitalize]{cleveref}

\crefname{section}{\S}{\S}
\crefname{table}{Tb.}{Tbs.}
\crefname{appendix}{App.}{Apps.}
\Crefname{theorem}{Thm.}{Thms.}
\Crefname{proposition}{Prop.}{Props.}
\crefname{algorithm}{Alg.}{Algs.}
\Crefname{assumption}{Asm.}{Asms.}
\crefname{mechanism}{Mech.}{Mechs.}
\Crefname{definition}{Def.}{Def.}

\crefformat{footnote}{#2\footnotemark[#1]#3}
\crefmultiformat{footnote}{#2\footnotemark[#1]#3}%
{\textsuperscript{,}#2\footnotemark[#1]#3}{\textsuperscript{,}#2\footnotemark[#1]#3}{\textsuperscript{,}#2\footnotemark[#1]#3}

\newcommand{\red}[1]{\textcolor{red}{#1}}
\newcommand{\green}[1]{\textcolor{OliveGreen}{#1}}
\newcommand{\orange}[1]{\textcolor{orange}{#1}}
\newcommand{\blue}[1]{\textcolor{cyan}{#1}}
\definecolor{mygray}{gray}{0.6}

\newcommand{\chulin}[1]{\textcolor{black}{#1}}

\newcommand{\myparatightestn}[1]{ \noindent\textbf{{#1}}}

\newcounter{packednmbr}

\NewDocumentCommand{\codeword}{v}{%
\texttt{\textcolor{blue}{#1}}%
}

\newcommand{\size}[2]{{\fontsize{#1}{0}\selectfont#2}}

\newcommand{\randomsampleapiname}{\size{9}{\textsf{RANDOM\_API}}}

\newcommand{\samplevariationapiname}{\size{9}{\textsf{VARIATION\_API}}}

\newcommand{\samplevariationapit}[1]

\newcommand{\dpvotingfunctionname}{\size{9}{\textsf{DP\_NN\_HISTOGRAM}}}

\newcommand{\embeddingnetworkname}{\Phi}

\newcommand{\priv}{\mathrm{priv}}
\newcommand{\syn}{\mathrm{syn}}

\newcommand{\yelp}{Yelp}
\newcommand{\openreview}{OpenReview}
\newcommand{\pubmed}{PubMed}
\def\calD{{\mathcal{D}}}

\SetCommentSty{mycommfont}
\usepackage{wrapfig,lipsum,booktabs}

\newcommand{\wbox}[1]{\setlength{\fboxsep}{1pt}\colorbox{yellow!50}{#1}}

\icmltitlerunning{Differentially Private Synthetic Data via Foundation Model APIs 2: Text}

\begin{document}

\twocolumn[
\icmltitle{Differentially Private Synthetic Data via Foundation
Model APIs 2: Text}

\icmlsetsymbol{equal}{*}

\begin{icmlauthorlist}
\icmlauthor{Chulin Xie}{uiuc}
\icmlauthor{Zinan Lin}{msr}
\icmlauthor{Arturs Backurs}{msr}
\icmlauthor{Sivakanth Gopi}{msr}
\icmlauthor{Da Yu}{sys}
\icmlauthor{Huseyin Inan}{msr}
\icmlauthor{Harsha Nori}{msr}
\icmlauthor{Haotian Jiang}{msr}
\icmlauthor{Huishuai Zhang}{msr}
\icmlauthor{Yin Tat Lee}{msr}
\icmlauthor{Bo Li}{uiuc,uchi}
\icmlauthor{Sergey Yekhanin}{msr}
\end{icmlauthorlist}

\icmlaffiliation{uiuc}{University of Illinois Urbana-Champaign}
\icmlaffiliation{uchi}{University of Chicago}
\icmlaffiliation{msr}{Microsoft Research}
\icmlaffiliation{sys}{Sun Yat-sen University}

\center {\footnotesize 
\texttt{chulinx2@illinois.edu},
\texttt{\{zinanlin,arturs.backurs,sivakanth.gopi,huseyin.inan, hanori,haotianjiang,huishuai.zhang,yintatlee,yekhanin\}@microsoft.com}, 
\texttt{yuda3@mail2.sysu.edu.cn},~\texttt{bol@uchicago.edu} 
}

\icmlkeywords{Machine Learning, ICML}

\vskip 0.3in
]

\printAffiliationsAndNotice{}  %

\begin{abstract}

Text data has become extremely valuable due to the emergence of machine learning algorithms that learn from it. A lot of high-quality text data generated in the real world is private and therefore cannot be shared or used freely due to privacy concerns.
Generating synthetic replicas of private text data with a formal privacy guarantee, i.e., differential privacy (DP), offers a promising and scalable solution. 
However, existing methods necessitate DP finetuning of large language models (LLMs) on private data to generate DP synthetic data. This approach is not viable for proprietary LLMs (e.g., \chatgpt{}) and also demands considerable computational resources for open-source LLMs.  
\citet{lin2023differentially} recently introduced the \emph{Private Evolution} (PE) algorithm to generate DP synthetic images with only API access to diffusion models.
In this work, \chulin{we propose an augmented  PE algorithm, named  \name{}, that applies to the complex setting of text}. We use API access to an LLM and generate DP synthetic text without any model training. We conduct comprehensive experiments on three benchmark datasets. Our results demonstrate that \name{}  produces DP synthetic text that yields competitive utility with the SOTA DP finetuning baselines.
This underscores the feasibility of relying solely on  API access of LLMs to produce high-quality DP synthetic texts, thereby facilitating more accessible routes to privacy-preserving LLM applications. Our code and data are available at \href{https://github.com/AI-secure/aug-pe}{https://github.com/AI-secure/aug-pe}.
\end{abstract}

\section{Introduction}
With recent advances in natural language processing (NLP), text-based applications have greatly facilitated our lives. These include AI-assisted medical record summaries \cite{rumshisky2016predicting}, email and document autocomplete tools \cite{voytovich2022natural,microsoftprivacy2023}, and personalized chatbots \cite{chew2022use}. However, all these applications (among others) rely on collecting private text data from users to train LLMs, which raises serious privacy concerns as LLMs may memorize and leak sensitive information about users \cite{carlini2021extracting,lukas2023analyzing,wang2023decodingtrust}. Differentially private synthetic text is a promising and actively studied solution \cite{putta2022differentially,bommasani2019towards}. It aims to create a new text dataset with similar characteristics to the original private data while ensuring privacy by protecting sensitive information in each sample (known as Differential Privacy (DP) \cite{dwork2014algorithmic}). The DP synthetic text can then be used in developing any downstream NLP system without adding extra privacy risks. It also allows the safe sharing of private data more broadly. For example, hospitals can share their private medical data for research purposes by creating a DP synthetic version of their data.
\begin{figure*}
\centering
    \includegraphics[width=1\linewidth]{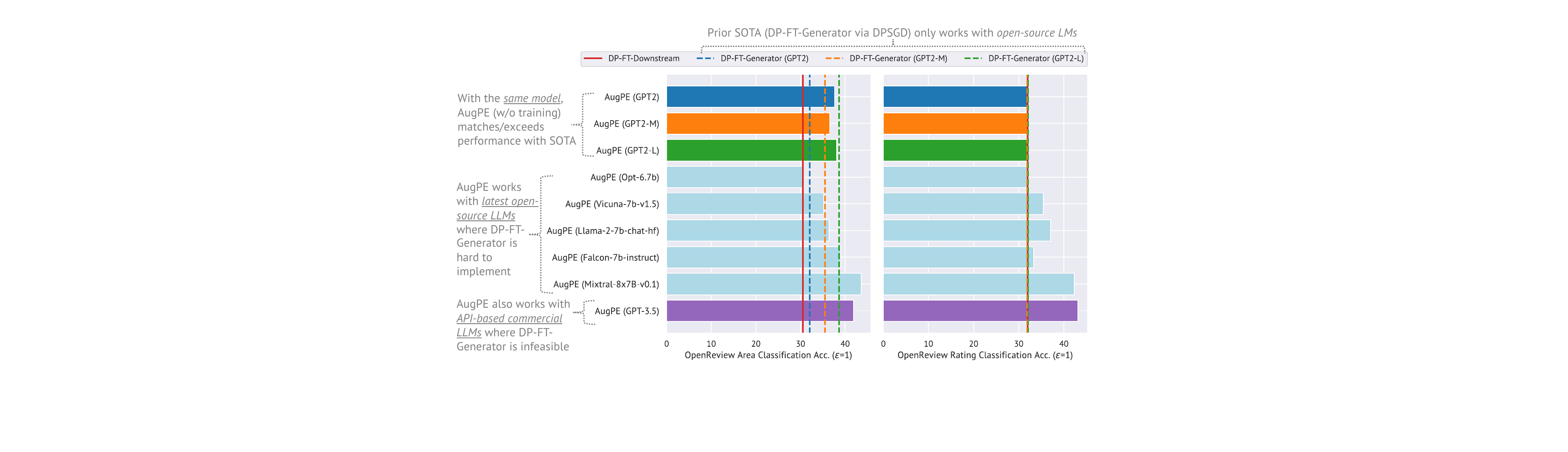} 
     \vspace{-6mm}
        \caption{\small Instead of finetuning LLMs with DP-SGD to generate synthetic text,  \name{} only requires inference APIs of LLMs. \name{} works with the latest open-source LLMs and API-based LLMs to generate DP synthetic text with improved utility on \openreview{} dataset, where DP-SGD finetuning is either hard to implement or infeasible.}
    \label{fig:compare}
    \vspace{-5mm}
\end{figure*}

The state-of-the-art DP synthetic text approach is to \emph{finetune pretrained generative language models (LMs) on private data} with DP-SGD (a DP variant of SGD~\cite{abadi2016deep})~\cite{yue2022synthetic,kurakin2023harnessing,mattern2022differentially} (short-handed as \emph{DP finetune generator}; see \cref{fig:compare}). Unlike non-DP ML applications, which have been greatly advanced by powerful LLMs such as \gptfour{} \cite{openai2023gpt4} and \llama{} \cite{touvron2023llama,touvron2023llama2} in a short time after they are released, the state-of-the-art DP synthetic text approaches are unfortunately still based on \gpttwo{}.\footnote{The 175 billion-parameter \gptthree{} has also been used for DP synthetic text \cite{he2022exploring}. However, the solution is not publicly accessible as \gptthree{} is proprietary.} The reasons are: (1) Many powerful LLMs such as \gptfour{}, \claude{}, and \bard{} are only accessible through APIs. DP finetuning them is not feasible.\footnote{Although standard finetuning APIs are provided for some of the models \cite{gpt35finetune}, DP finetuning requires a special implementation and no model provides this custom API to date.} (2) Even though some LLMs (e.g., LLaMA) are open-source, finetuning them with DP is resource-intensive and non-trivial to implement due to the need to calculate \textit{per-sample} gradients (see \cref{sec:background}).

A recent DP synthetic data framework called Private Evolution (\pename{}) \cite{lin2023differentially} offers a new opportunity to circumvent these challenges by only requiring API access to foundation models, without needing any model training. The high-level idea is to first draw random samples from a foundation model, and then iteratively improve them by selecting (with DP) the most similar ones to the private dataset and querying foundation models to generate more of such samples. \pename{} shows promising results on \emph{images} by leveraging pretrained Diffusion Models~\cite{rombach2022high}: in certain cases, \pename{} achieves an even better privacy-utility trade-off than DP finetuned generators \cite{lin2023differentially}.

However, extending \pename{} to text is highly non-trivial. 
\pename{} requires APIs that generate random samples and variations of a given sample, which need to be \textit{redesigned} for text. \chulin{In particular, unlike generating image variants in the continuous pixel space where diversity 
can be easily manipulated using existing model hyperparameters (e.g., guidance scale in diffusion model~\cite{ho2021classifier}), texts operate in a discrete space, making it challenging to effectively \textit{control} the generation diversity.  In addition, in contrast to images with fixed dimensionality, text data exhibit varied \textit{lengths} which adds another layer of complexity.  To this end, we propose an augmented \pename{} algorithm (\name{}) with \textit{new \textbf{generation} and \textbf{selection} techniques} that allow us to i) elicit a larger set of more diverse and higher-quality texts from LLMs with appropriate sequence length and ii) effectively select the most relevant texts.}
Our contributions are: %
\begin{itemize}[noitemsep,leftmargin=*]
    \vspace{-4mm}
    \item %
    We propose \name{} for high-quality DP synthetic text generation leveraging API access to powerful LLMs. This includes both a practical instantiation of \pename{} on texts and fundamental algorithmic innovations that may benefit future applications of \pename{}. %
    \item  We conduct  comprehensive evalutions of \name{} on \yelp{}, \openreview{} (ICLR 2023), and \pubmed{} (Aug 2023) datasets with various LLMs, including  \gpttwo{}-series models,  \gptthreepointfive, and  open-source LLMs. We
    show that under \emph{the same pretrained LM} (\gpttwo{}-series) and privacy budget $\epsilon=4,2,1$, \name{} can generate DP synthetic text that achieves comparable or even better performance than finetuning baselines in some cases, in terms of downstream task utility and similarity between synthetic and real samples. 
    Leveraging \textit{more powerful LLMs} such as \gptthreepointfive{} (where DP finetuning is not applicable) and five open-source LLMs (where DP finetuning is hard to implement), the performance of \name{} can be significantly improved. 
    Additionally, \name{} can be more computationally efficient than DP finetuning by requiring LLM \emph{inference} APIs only. 
    \item We explore the properties of \name{} including its text length distribution,  its compatibility with stronger LLMs as data generators and downstream models, and its behaviors under data scaling, to provide insights for future development of \pename{}.
\end{itemize}

\vspace{-5mm}
\section{Background}
\label{sec:background}

\myparatightestn{Differential Privacy (DP).} 
$(\epsilon, \delta)$-DP ensures that the output of a randomized mechanism $\mathcal{M}$ is close regardless of whether an individual data record is included in the input or not. Specifically, given any pair of two adjacent datasets $\calD, \calD'$ (i.e., adding or removing one sample), any possible output set $E$, it holds that $ \operatorname{Pr}[\mathcal{M}(\calD) \in E] \leq e^{\epsilon} \operatorname{Pr}\left[\mathcal{M}\left(\calD^{\prime}\right) \in E\right]+\delta.$
Moreover, arbitrary post-processing of the output of an $(\epsilon,\delta)$-DP mechanism does not incur additional privacy loss, based on the \textit{post-processing property} of DP~\cite{dwork2014algorithmic}.

\myparatightestn{DP synthetic text.}
To guarantee DP for private training data, one method involves using DP-SGD~\cite{abadi2016deep} during model training for specific NLP tasks~\cite{yu2022differentially,li2021large}. Alternatively, one can finetune pretrained generative language models, such as GPT-2, with private data using DP-SGD and then generate synthetic text datasets~\cite{putta2022differentially,bommasani2019towards} (\cref{fig:compare}). Such DP synthetic texts can be employed in an arbitrary number of non-privately trained downstream tasks without increasing privacy loss. Studies by \citet{yue2022synthetic,mattern2022differentially,kurakin2023harnessing} indicate that training downstream models on DP synthetic text yields performance akin to directly training them on real data with DP, highlighting the good quality of synthetic data.

However, given that state-of-the-art LLMs (e.g., \gptfour{}, \claude{}, \gptthreepointfive{}) do not provide model weights, DP finetuning them is infeasible.  Even for open-source LLMs (e.g., \llama{} \cite{touvron2023llama,touvron2023llama2}), it is resource-intenstive to perform finetuning \cite{malladi2023fine}. 
Finetuning with DP-SGD is even harder due to the well-known challenges of \textit{per-sample gradient} calculations for clipping to guarantee DP. 
Even with optimization techniques \cite{malladi2023fine,he2022exploring}, DP finetuning is still memory and computationally intensive due to large batch sizes and long training iterations required to reach a good fidelity-privacy trade-off \cite{anil2021large}. 
Here, we study an API-based method for DP synthetic text generation to overcome these challenges, \chulin{which only requires model inference and is applicable no matter whether the LLM is open-sourced or not.}

Additionally, there is a line of work on text-to-text privatization techniques, which provide different privacy guarantees than DP, such as word-level metric DP or sample-level local DP. We defer more discussion and comparison to  \cref{app:text_to_text_privatization}.

\section{Method}
\subsection{Preliminaries on Private Evolution (PE)}
\pename{} is recently proposed as an alternative to DP finetuning for DP synthetic data generation \cite{lin2023differentially} by merely requiring APIs of pretrained models, and thus is easier to implement and deploy and can leverage API-based models. The original \pename{} algorithm (for unconditional generation)\footnote{The conditional version of PE is running \cref{algo} for the private samples from each class/label separately; see \citet{lin2023differentially}.} is the $L=1$ %
case in \cref{algo}. 
\pename{} works by first calling \randomsampleapiname{} that generates random samples from the foundation model (\cref{line:random}), and then iteratively: (1) using private samples to vote for their nearest synthetic samples (under embedding model $\embeddingnetworkname{}$) to construct a \dpvotingfunctionname{} (\cref{line:gethistogram}), (2) drawing samples according to the histogram (\cref{line:drawfromhist}), and (3) passing those samples through \samplevariationapiname{} which generates new samples that are similar to the given one (\cref{line:variation}), e.g., images with a similar object.  %

 While the PE framework is general across modalities, its core components including $\embeddingnetworkname{}$ (the embedding model), \randomsampleapiname{} (API for generating random samples from the pretrained model), and \samplevariationapiname{} (API for generating new samples that are similar to the given one) require domain-specific designs, and the original paper \cite{lin2023differentially} only explores their implementation for images. Compared to images, text introduces unique challenges. For example, unlike images which have a fixed dimensionality, the length of text can vary. %
 In addition, the original \pename{} algorithm yields unsatisfactory text quality.
In the following, we explore our design choices for each component and propose our augmented version on text,  \name{} (shown in \cref{algo} and \cref{fig:alg_overview}) with new algorithmic techniques to increase the diversity and quality of text generation.

\begin{figure}
    \centering
    \includegraphics[width=1.\linewidth]{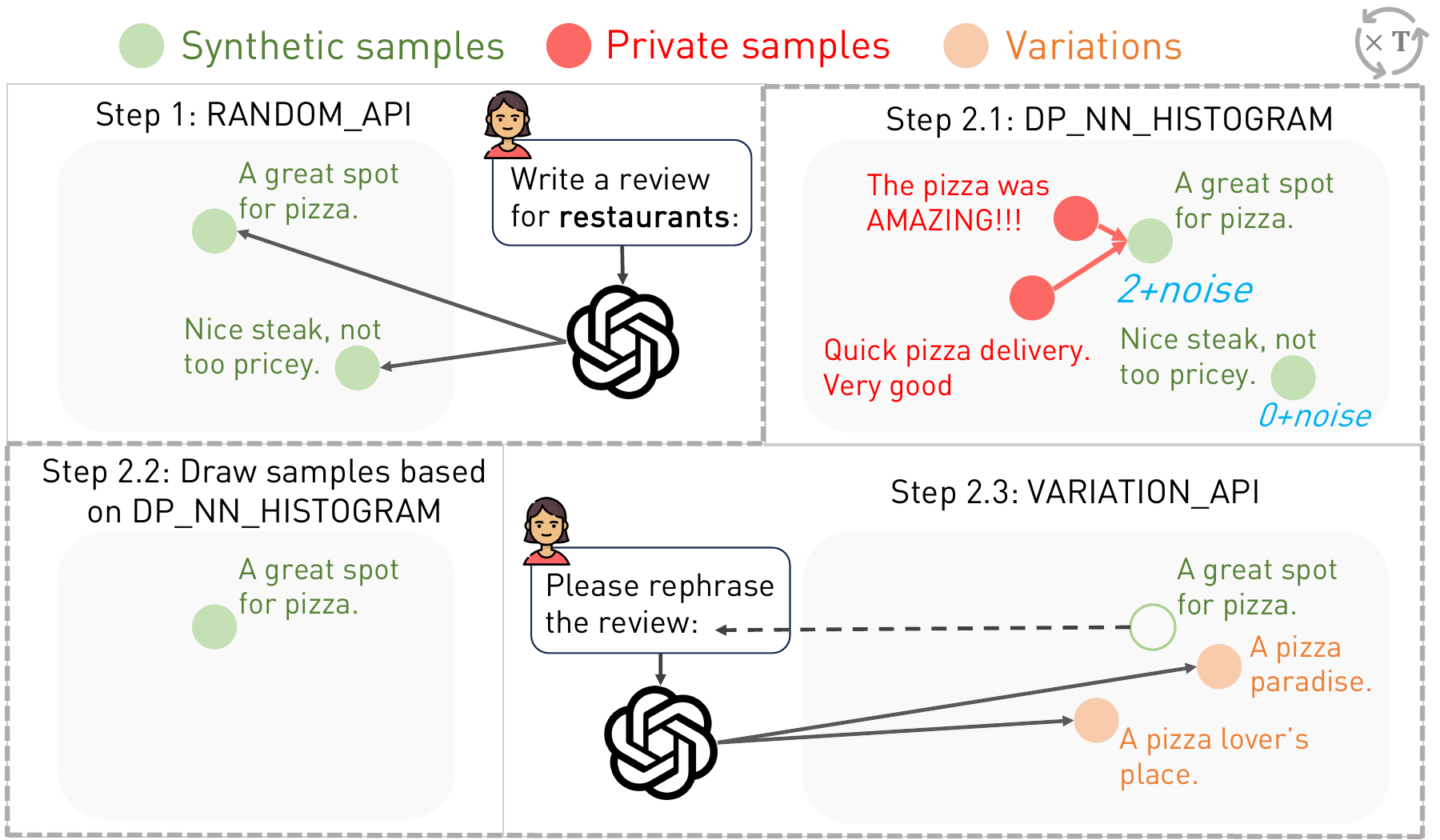}
    \vspace{-6mm}
    \caption{\small Overview of \name{}. We use two private \& synthetic samples (reviews for the ``restaurant'' class) for illustration. %
    \textbf{Step 1} (\randomsampleapiname{}, \cref{line:random}): we use prompts to generate \green{random samples} from the LLM. \textbf{Step 2}: we iteratively go through steps 2.1-2.3 to refine the \green{synthetic samples} towards the \red{private samples}. \textbf{Step 2.1} (\cref{line:gethistogram}): each \red{private sample} votes for their closet \green{synthetic sample} (using self-embedding~\cref{line:selfemb} or mean embedding \cref{line:meanemb}) in the embedding space induced by embedding model $\embeddingnetworkname{}$. ``A great spot for pizza'' gets 2 votes, and the other sample gets 0 votes. We then add Gaussian noise to the votes to ensure DP. This gives us the \blue{DP Nearest Neighbor Histogram} (\dpvotingfunctionname{}). \textbf{Step 2.2}: we resample the generated texts according to the histogram. We assume that only ``A great spot for pizza'' remains. \textbf{Step 2.3} (\samplevariationapiname{}): we use prompts to ask the LLM to generate \orange{new similar samples}, which are the initial \green{synthetic samples} in the next iteration.
    The prompts are simplified for illustration; see \cref{app:exp-details} for the complete prompts. 
  }
    \label{fig:alg_overview}
    \vspace{-4mm}
\end{figure}

\vspace{-1mm}
\subsection{\name{} Design}
\label{sec:pe_design}
\myparatightestn{\randomsampleapiname{}.} Given the strong instruction-following capability of LLMs, we consider directly using prompts to generate samples (step 1 in \cref{fig:alg_overview}). Following \citet{yue2022synthetic}, we assume that class labels are non-private. Therefore, we put class label in the prompt (e.g., ``restaurant'' in \cref{fig:alg_overview}).  
To encourage diverse generation, we propose a \emph{pseudo-class} approach, where we generate a list of subcategories for each class from \chatgpt{} and randomly sample one subcategory as the keyword to put in the prompt for each generation (e.g., \textit{Steakhouse},
\textit{Bistros} for restaurants).

\myparatightestn{\samplevariationapiname{}} 
 takes %
 a sample as input and outputs its variations.\footnote{While the function processes each sample independently, for notation simplicity, we input an entire dataset to \samplevariationapiname{}, which outputs corresponding variations for each sample within it.}
Unlike image diffusion models used in \citet{lin2023differentially}, 
text models usually do not provide off-the-shelf variation APIs. Again, we leverage the instruction-following capability of LLMs to implement this via prompting.
We propose two variation methods: \textit{paraphrasing} and \textit{fill-in-the-blanks}.
For \textit{paraphrasing}, we use the prompt ``Please rephrase the below sentences: $\{\mathrm{input}\}$''. 
For \textit{fill-in-the-blanks}, we mask $p\%$ tokens of $\mathrm{input}$ as blanks, resulting in $\mathrm{masked\_input}$, and use ``Please fill in the blanks for the below sentences: $\{\mathrm{masked\_input}\}$'' as the prompt.
Given the in-context learning ability of recent LLMs, we provide \textit{few-shot demonstrations} to improve the generation quality. %
To add diversity to the generated variations, we create \textit{tone} candidates (e.g.,  ``in a creative way'', ``in a professional style''), randomly subsample one tone, and add such phrase into the prompt for each generation.

    \begin{minipage}{0.5\textwidth}
    \vspace{-5mm}
      \begin{algorithm}[H]
      \scriptsize
        \caption{\footnotesize \highlightbox{Augmented} Private Evolution (\name)}
\label{algo}
        \KwIn{ private dataset $S_{\mathrm{pri}}$, noise multiplier $\sigma$, text embedding model  $\Phi$, number of synthetic samples  $N_\mathrm{syn}$, $K$, $L$}
    \KwOut{Synthetic text dataset ${S_\mathrm{syn}}_{T}$}
     \SetKwFunction{funone}{DP\_NN\_HISTOGRAM}
     $E_\mathrm{pri} = \Phi (S_{\mathrm{pri}})$\label{line:priemb}\;  

    $S_0 \gets  $\highlightbox{$\mathrm{RANDOM\_API}$} ($N_\mathrm{syn}$\highlightbox{$*\mathrm{L}$})\label{line:random}\;
    
    \For{iteration  $t=0$ \KwTo $T-1$}
        {  
           
           \tcp{\scriptsize  embedding calculation for  synthetic samples }
             \tikzmk{A}
             \uIf{  $K==0$}{$E_t=\Phi (S_t) $  \label{line:selfemb}}
               \ElseIf{ $K>0$}{
                $S_t^k \gets $\highlightbox{$\mathrm{VARIATION\_API}$} $ (S_t)$ for $k =1,2 \ldots, K$
                
           $E_t= \frac{1}{K}\sum_{k=1}^K   \Phi (S_t^k)$ \label{line:meanemb}
           }
          \tcp{\scriptsize DP histogram calculation} 
           $\mathrm{Histogram}_t \gets  \funone(E_t,E_\mathrm{pri}, \sigma )$\label{line:gethistogram}\;
           
           $P_t \gets $  $\mathrm{Histogram}_t$ / sum ($\mathrm{Histogram}_t$)\;

            \tcp{\scriptsize synthetic sample selection and generation} 
           \uIf{$L==1 $}{
                $S_t'\gets$ draw $N_\mathrm{syn}$ samples with replacement from $S_t$ with probability $P_t$\label{line:drawfromhist}

                 $S_{t+1} \gets $\highlightbox{$\mathrm{VARIATION\_API}$}$ (S_t')$   \label{line:variation}

                save dataset ${S_\mathrm{syn}}_{t+1} \gets S_{t+1}$ 
            }
           \tikzmk{A}
           \ElseIf{$L>1$}{ 
                    $S_t'\gets$  \textbf{rank samples} by probabilities  $P_t$ and draw top $N_\mathrm{syn}$ samples \label{line:directuse}

                    save dataset   ${S_\mathrm{syn}}_{t+1} \gets S_t'$ \label{line:savestepinllargerthanone}

                      $S_{t+1}^{j} \gets \mathrm{VARIATION\_API} (S_t') $ for \textbf{$\mathbf{j=1, 2 \ldots, L-1}$  } \label{line:variationexpanded}

                   $S_{t+1} \gets [S_{t+1}^1, ..., S_{t+1}^{L-1}, \mathbf{S_t'}]$ \label{line:expandedsample}
               }   
           \tikzmk{B}\boxit{mygray} 
        }
    \KwRet{ ${S_\mathrm{syn}}_{T}$}

   \SetKwProg{fooo}{Procedure}{}{}
    \fooo{\funone{$E_\mathrm{syn},E_\mathrm{pri}, \sigma$}}{
    \KwIn{ 
    synthetic embedding set $E_\mathrm{syn} =\{e_j\}_{j=1}^{n}$,
    private embedding set $E_\mathrm{pri}$, 
    noise level $\sigma$, distance function $d(\cdot,\cdot)$}
    
     $\mathrm{Histogram} \gets [0,...,0]$  %
     
     \For{ $e_{\mathrm{pri}} \in E_\mathrm{pri}$ }
            { 
            $i= \arg\min_{j \in [n]} d(e_{\mathrm{pri}}, e_j ) $;
            
            $\mathrm{Histogram}[i] \gets \mathrm{Histogram}[i] +1$
           }
      $\mathrm{Histogram} \gets \mathrm{Histogram}+  \mathcal{N}(0, \sigma^2 I_{n})$ \label{line:addnoise}
      
       \KwRet{$\mathrm{Histogram}$}
    }
      \end{algorithm}
    \end{minipage}

\myparatightestn{Adaptive text lengths in \samplevariationapiname{}.} 
The distribution of text length in real-world datasets is usually fat-tailed: most samples are short while a few are long (\cref{fig:length-compare}). In DP-finetuning-based approaches, to faithfully capture long texts, we need to set a large max token length (denoted by $\mathrm{max\_token}$). However, this would significantly increase the computation cost. Prior work \cite{yue2022synthetic} circumvents this problem by setting a \emph{small} $\mathrm{max\_token}$ at the cost of the capability to generate long texts.
\name{} faces the same challenge. 
Since APIs usually charge by token usage, a high $\mathrm{max\_token}$ raises costs (as generated text can exceed needs), while a low $\mathrm{max\_token}$ sacrifices fidelity. 

To address the challenge, we leverage \pename{} to learn text lengths automatically by adjusting \emph{per-sample} $\mathrm{max\_token}$ adaptively. Specifically, in  \samplevariationapiname{}, we add %
``with $\{\mathrm{targeted\_word}\}$ words'' in the prompt to specify the desired word count in the %
generation. 
$\mathrm{targeted\_word}$ is modified by
setting $\mathrm{targeted\_word} = \max\{\mathrm{original\_word} +\mathcal{N}(0, \sigma_{word}^2), \mathrm{\min\_word}\}$ 
where $ \mathrm{original\_word}$ is the %
word count of
$\mathrm{input}$, $\sigma_{word}^2$ is  Gaussian noise variance
and $\mathrm{\min\_word}$ is a minimal targeted word ensuring useful generations. 
We set $\mathrm{max\_token}=\lfloor \mathrm{targeted\_word}* \mathrm{w2t\_ratio} \rfloor$ for LLM API calls where $\mathrm{w2t\_ratio}$ is the approximate number of tokens per word%
~\cite{openaitoken2word}.

\myparatightestn{Embeddings calculation and DP nearest neighbor histogram.}
We use off-the-shelf text embedding models $\Phi$ %
to calculate the embedding of private/synthetic samples. 
Notably, the embedding of synthetic samples can be defined either by their self-embedding (when $K=0$) or the averaged embedding from $K$ variations (when $K>0$).
After calculating embeddings, %
each private sample votes for its nearest synthetic sample in the embedding distance, which results in the $\mathrm{Histogram}_t$ for synthetic samples. 
As the voting utilizes private samples, we add Gaussian noise $\mathcal{N}(0, \sigma^2)$ to each bin of $\mathrm{Histogram}_t$ to ensure DP.

\myparatightestn{Sample selection and generation.}
\name{} introduces significant enhancements over the original \pename{}  for generating more diverse samples and selecting/retaining high-quality samples. Specifically, to \emph{enhance sample diversity}, we propose  the following methods:
(1) The random sampling based on the histogram probability $P_t$ (\cref{line:drawfromhist}) in original \pename{} results in repeated samples, causing performance degradation for $S_t'$. To mitigate this, \name{} ranks synthetic samples according to their probability and selects only the top $N_\mathrm{syn}$ samples, enhancing the diversity without sample redundancy (\cref{line:directuse}).
(2) Instead of a single variation, \name{} generates $L-1$ variations for each selected sample in  $S_t'$, creating a larger and more diverse synthetic dataset $S_{t+1}$ for subsequent iterations (\cref{line:variationexpanded}).
(3) We modify the size of the initial dataset to be $L$ times larger than $N_{\mathrm{syn}}$, matching the expanded size of $S_{t+1}$ (\cref{line:random}).
To \emph{select/retain high-quality samples}, we propose the following methods:
(1) The selected samples $S_t'$ are also included in the next iteration's dataset $S_{t+1}$, increasing the likelihood of retaining high-quality synthetic candidates (\cref{line:expandedsample}). 
(2) For LLMs, we find that when the variation API produces samples with large variations, the averaged embedding from the variations is not representative of the actual sample. Therefore, we use $K=0$ so the nearest neighbor voting is performed on the self-embedding of synthetic samples and we directly use those selected, good samples as algorithm's output ${S_\mathrm{syn}}_{t+1} \gets S_t'$ (\cref{line:savestepinllargerthanone}).  

In practice, we use  $\{K=\text{\#variations}, L=1\}$ as original \pename{}, and $\{K=0, L=\text{\#variations}+1\}$ as \name{}, so that the number of API calls for generating variations (i.e., \#variations)  are kept the same for fair comparisons. 

These enhancements position \name{} as a more effective method to generate diverse and high-quality synthetic \textit{text}.

\textbf{Privacy analysis of \name{}}
follows original \pename{} and we provide detailed privacy analysis in \cref{app:privacy-analysis}. Specifically,  since each private sample only contributes 1 vote for one bin in the histogram (i.e., nearest synthetic sample), the sensitivity is 1. The histograms are privatized by adding Gaussian noise. The adaptive DP composition theorem~\cite{dong2019gaussian} is applied to track the privacy loss across $T$ iterations.

\begin{table*}[h]
    \caption{\small Evaluation on downstream model accuracy of three methods along 4 data generators. The highest accuracy across all methods (\colorbox{lightgray!70}{obtained by \name{}})  is \textbf{bolded} (\underline{underlined}). (i) Compared to \ftgenerator{}, 
    in some cases, downstream accuracy of \name{} is higher  (\ua{})  under the same size of \gpttwo-series data generator. Leveraging the inherent knowledge within stronger LLM, \gptthreepointfive, \name{} can achieve higher accuracy, especially on challenging datasets \openreview{} and \pubmed{}, outperforming \ftgenerator{} by a notable margin. (ii) Compared to traditional method \ftdownstream{}, \name{} can also obtain higher accuracy under DP.
}
\vspace{-2mm}
\label{tb:main_result}
\resizebox{\linewidth}{!}{%
\begin{tabular}{llllll|ll|ll|llll}\toprule
\textbf{Dataset}  & \textbf{Method} & \textbf{Data Type (Size)} &  \textbf{Data Generator} & \multicolumn{2}{c}{$\epsilon=\infty$} & \multicolumn{2}{c}{$\epsilon=4$} & \multicolumn{2}{c}{$\epsilon=2$} & \multicolumn{2}{c}{$\epsilon=1$}    \\ \midrule
 \multirow{12}{*}{\yelp{}} &   &  &  &  \yelprating & \yelpcategory  &  \yelprating & \yelpcategory   &  \yelprating & \yelpcategory    &  \yelprating & \yelpcategory   \\
 \cmidrule(lr){5-12}
 & \multirow{2}{*}{\ftdownstream{}} & Original (1939290 / full data) & \multirow{2}{*}{-}  &  \textbf{76.0} & \textbf{81.6} & 67.5 & 72.8 & 67.2 & 72.0 & 66.8 & 71.8 \\
&  &  Original (5000)  &      &   70.5 & 75.1 & 44.8 & 61.8 & 44.8 & 61.8 & 44.8 & 61.8 \\
\cmidrule(lr){2-12}
&  \multirow{3}{*}{\ftgenerator{}} &  \multirow{3}{*}{Synthetic (5000)}  & \gpttwo&  70.3 & {75.9} & 68.2 & 74.1 & 67.2 & 73.1 & 66.4 & 73.9  \\
&  & &   \gpttwom   & 70.0 & 75.0 & \textbf{69.0} & 74.6 & 67.8 & 74.3 & 67.4 & 74.1  \\
& &  &  \gpttwol  & 70.4 & 75.4 & 68.7 & 74.2 & \textbf{69.8} & \textbf{75.1} & \textbf{68.7} & 74.6  \\
\cmidrule(lr){2-12}
\rowcolor{lightgray!60}
\cellcolor{white}&  &  &  \gpttwo  &  67.5 & 74.8 & 66.4 & \underline{\textbf{74.9}} \ua{} & 67.1 & 74.7 \ua{} & 66.9 \ua{} & 74.4 \ua{} \\
\rowcolor{lightgray!60}
\cellcolor{white}&  & &  \gpttwom   & 67.5 & \underline{74.9} & 66.8 & 74.6 & 67.7 & \underline{74.7} \ua{} & 67.3 & 74.6 \ua{} \\
\rowcolor{lightgray!60}
\cellcolor{white}& &  & \gpttwol &  67.5 & 74.5 & 67.3 & 74.4 \ua{}  & 65.8 & 74.1 & 66.5 & \underline{\textbf{75.0}} \ua{} \\ 
\rowcolor{lightgray!60}
\cellcolor{white}&  \multirow{-4}{*}{\name} &  \multirow{-4}{*}{Synthetic (5000)}  &  \gptthreepointfive    &   \underline{68.4} & 74.1 & \underline{68.1} & 74.0 & \underline{67.8} & 74.3 & \underline{67.9} & 74.0\\ 
\midrule\midrule
\multirow{12}{*}{\openreview{}} &  &  &  &  \openreviewarea & \openreviewrecom  &  \openreviewarea & \openreviewrecom   &  \openreviewarea & \openreviewrecom   &  \openreviewarea & \openreviewrecom \\
\cmidrule(lr){5-12}
& \multirow{2}{*}{\ftdownstream{}} & Original (8396 / full data) & \multirow{2}{*}{-}  & \textbf{65.1} & \textbf{50.8} & 30.5 & 32.0 & 30.5 & 32.0 & 30.5 & 32.0 \\
&  & Original (2000)  &   &  55.3 &  47.8 & 30.5 & 32.0 & 30.4 & 25.5 & 6.3 & 19.8 \\ 
\cmidrule(lr){2-12}
&  \multirow{3}{*}{\ftgenerator{}} & \multirow{3}{*}{Synthetic (2000)} & \gpttwo  & 47.5 & 32.0 & 32.1 & 32.0 & 31.9 & 32.0 & 32.1 & 32.0 \\
&  & & \gpttwom   & 49.7 & 36.5 & 40.3 & 32.0 & 33.5 & 31.9 & 35.5 & 31.9 \\
&  & & \gpttwol & 48.3 & 42.9 & 38.9 & 33.7 & 40.4 & 33.6 & 38.6 & 32.1 \\
\cmidrule(lr){2-12}
    \rowcolor{lightgray!60}
\cellcolor{white} &  &  & \gpttwo  & 42.4 & 32.1 \ua{}  & 39.9 \ua{} & 32.1 \ua{} & 38.8 \ua{} & 32.1 \ua{} & 37.6 \ua{}  & 32.0 \\
    \rowcolor{lightgray!60}
\cellcolor{white}&  &  & \gpttwom   & 41.0 & 32.3 & 36.9 & 32.0 & 36.0 \ua{} & 32.0 \ua{} & 36.6 \ua{} & 32.1 \ua{} \\
    \rowcolor{lightgray!60}
\cellcolor{white}&  &  & \gpttwol  & 42.1 & 32.1 & 38.8 & 32.0 & 38.4  & 32.0 & 38.1 & 32.0 \\
    \rowcolor{lightgray!60}
\cellcolor{white}&   \multirow{-4}{*}{\name} & \multirow{-4}{*}{Synthetic (2000)}  & \gptthreepointfive      &  \underline{45.4}  & \underline{43.5} & \underline{\textbf{43.5}} & \underline{\textbf{44.6}} & \underline{\textbf{42.8}} & \underline{\textbf{44.5}} & \underline{\textbf{41.9}} & \underline{\textbf{43.1}}  \\
\midrule\midrule
\multirow{12}{*}{\pubmed{}} &   &   &  &  \bertmini & \bertsmall   &  \bertmini & \bertsmall   &  \bertmini & \bertsmall    &  \bertmini & \bertsmall \\\cmidrule(lr){5-12}
& \multirow{2}{*}{\ftdownstream{}} & \multirow{1}{*}{Original (75316 / full data)} & \multirow{2}{*}{-}  & \textbf{43.5} & \textbf{47.6} & \textbf{30.7}  & \textbf{34.1} & 28.9 & \textbf{32.5} & 26.7  & 30.4 \\
& & \multirow{1}{*}{Original (2000)} &  & 33.5 & 34.6  & 2.2 & 1.1 &  1.8 & 0.8 &  1.4 & 0.6 \\
\cmidrule(lr){2-12}
&  \multirow{3}{*}{\ftgenerator{}} & \multirow{3}{*}{Synthetic (2000)} &  \gpttwo & 30.2 & 32.4 & 27.8 & 29.7 & 27.6 & 29.3 & 27.2 & 29.2  \\
& & &  \gpttwom &31.0 & 33.1 & 28.4 & 30.2 & 28.1 & 30.0 & 27.8 & 29.8  \\
& &  & \gpttwol & 31.0 & 33.1 & 29.2 & 31.2 & 29.2 & 31.1 & 28.9 & 31.1 \\
\cmidrule(lr){2-12}
    \rowcolor{lightgray!60}
\cellcolor{white}&  &  & \gpttwo & 24.5 & 26.7 & 24.7 & 27.0 & 24.7 & 26.9 & 24.3 & 26.5  \\
    \rowcolor{lightgray!60}
\cellcolor{white}& & & \gpttwom & 25.5 & 27.7 & 25.4 & 27.6 & 25.1 & 27.4 & 24.9 & 27.0 \\
    \rowcolor{lightgray!60}
\cellcolor{white}& & & \gpttwol & 25.7 & 28.0 & 25.8 & 27.9 & 25.5 & 27.7 & 25.1 & 27.2  \\
\rowcolor{lightgray!60}
\cellcolor{white} & \multirow{-4}{*}{\name} & \multirow{-4}{*}{Synthetic (2000)}  &  \gptthreepointfive     & \underline{30.4} & \underline{32.7} &  \underline{{30.3}} & \underline{{32.5}} &  \underline{\textbf{30.2}} &  \underline{\textbf{32.5}}  & \underline{\textbf{30.1}} & \underline{\textbf{32.4}}    \\
 \bottomrule
\end{tabular}
}
\vspace{-3mm}
\end{table*}

\section{Experiments}
\myparatightestn{Datasets.} We evaluate \name{} on three datasets: \yelp{} Review~\cite{yelpdataset}, \openreview{}, and \pubmed{} abstracts.
We use \yelp{}, a public benchmark providing reviews on businesses, following the choice in prior work for DP synthetic text~\cite{yu2022differentially}.
To mitigate the concerns that existing benchmarks are potentially used at LLM's pretraining stage, we crawl the latest reviews for ICLR 2023 submissions from \openreview{} website\footnote{https://openreview.net/group?id=ICLR.cc/2023/Conference} to construct a new dataset, where the reviews are made public after recent LLMs are trained.  We also use \pubmed{} with abstracts of medical papers\footnote{https://www.ncbi.nlm.nih.gov/} crawled by \citet{yu2023training} from 2023/08/01 to 2023/08/07 after recent LLMs are trained. %
Notably, texts from \yelp{} are mainly in styles of daily conversation, while
the other two datasets require domain-specific knowledge about machine learning or biomedical literature when generating DP synthetic replicas. 
For conditional generation, we use below attributes as labels: the review ratings and business category for \yelp{}, and the review recommendation and area for \openreview{}.
For \pubmed{}, we use unconditional generation.

\myparatightestn{Models.} 
For data generators, we use 
 \gpttwo{}~\cite{radford2019language}, \gpttwom{}, \gpttwol{},  \chatgpt{}~\cite{chatgpt}, and non-GPT based LLMs
including four 7b-sized models -- OPT~\cite{zhang2022opt}, Vicuna~\cite{zheng2023judging}, Falcon~\cite{falcon40b}, \llamatwo{} -- as well as one Mixture-of-Expert model Mixtral-8x7B~\cite{mistral}. 
For embedding models, we use sentence-transformer~\cite{reimers2019sentence}. %
\emph{We study more types of embedding models as ablation study in \cref{sec:exp_understand_property}.}

\myparatightestn{Metrics.}
We evaluate synthetic texts regarding (i) accuracy on downstream tasks, and (ii) similarity between real and synthetic data. 

\textit{Downstream tasks:} we finetune downstream models on the synthetic text and evaluate their accuracy on the real test dataset. 
We pick two representative use cases: using DP synthetic text to train DP text classifiers~\cite{yue2022synthetic} and to train efficient DP lanaguge models~\cite{yu2023training}. 
Specifically, we finetune \robertabase{}~\cite{liu2019roberta} as text classifiers to classify review ratings and business categories for \yelp{}, and to classify review recommendations and areas for \openreview{}. 
For \pubmed{}, we finetune 
\bertmini{}/\bertsmall{}~\cite{turc2019well}\footnote{Following \cite{yu2023training}, we apply a causal
language modeling mask that restricts each token to only attend
to its preceding tokens.} on synthetic text and evaluate their next-word prediction accuracy. 
\emph{We study more types downstream models as ablation study in \cref{sec:exp_understand_property}.}

\textit{Similarity between real and synthetic data:} we quantitively compare 
(a) \textit{embedding distribution distance} (i.e., Fréchet Inception Distance ({FID})~\cite{heusel2017gans},
{Precision}, {Recall}, {F1} score~\cite{kynkaanniemi2019improved},
MAUVE score~\cite{pillutla2021mauve},
{KL} and 
{TV} divergences~\cite{chung1989measures}) and qualitatively compare  (b) \textit{text length distribution difference}~\cite{yue2022synthetic}.

\myparatightestn{Baselines.}
We consider two SOTA baselines {involving DP finetuning}: (1) \ftdownstream{}~\cite{yu2022differentially,li2022large}: finetuning downstream model on real data with DP-SGD. Note that this baseline is not a competitor to our method, since \textit{our goal is to generate DP synthetic data and not merely train a downstream model}. (2) \ftgenerator{}~\cite{yue2022synthetic}: finetuning generator (e.g., \gpttwo{}) with DP-SGD (note that we cannot finetune closed-source \chatgpt{}) and using synthetic texts to finetune downstream model {with non-private SGD.}

We defer more details about the setups, hyperparameters and metrics to \cref{app:exp-details}.

\subsection{Understanding the Performance of \name{}}
Here, we analyze the performance of \name{} by answering five research questions about its utility, efficiency, and robustness against empirical privacy attacks under DP compared to DP-finetuning-based baselines.

\textit{RQ1: Can DP synthetic texts generated from \name{} outperform  those from \ftgenerator{}? }
\textbf{DP synthetic texts from \name{} can have comparable privacy-utility trade-off to those from \ftgenerator{} using the same generator, while outperforming it using the stronger generator  \gptthreepointfive{}.} 
The downstream model accuracy of different methods along 4 generators on different benchmark datasets is shown in \cref{tb:main_result}. 
\textbf{(1)} When using the same LM (\gpttwo{}-series) as the generator for fair comparisons, DP synthetic texts from \name{} demonstrate competitive or even better  (\ua{}) utility than \ftgenerator{} on \yelp{} and \openreview{}.
However, \name{} underperforms \ftgenerator{} on \pubmed{}. This is expected because \name{} relies on the knowledge within LLMs to generate high-quality texts without domain-specific finetuning, while \gpttwo{}-series models might have limited exposure to biomedical literature~\cite{radford2019language}. 
\textbf{(2)} \name{} only requires API access,  making it possible to use closed-source LLM such as \gptthreepointfive{} for generating DP synthetic text. 
The results of \gptthreepointfive{} outperform not only \name{}  \gpttwo{}-series, but also \ftgenerator{} \gpttwo{}-series by a significant margin,  especially on challenging datasets such as \openreview{} and \pubmed{}.  It shows that \name{} can effectively leverage the inherent knowledge (e.g., medical knowledge, sentiment of reviews, research areas about machine learning) in stronger LLMs to generate higher-quality DP synthetic texts.
\textbf{(3)} In addition to downstream utility, we measure the \textit{embedding distribution distance} between real and synthetic samples. 
The results in \cref{app:emb_distance} show that \name{} can obtain similar and even lower distances (reflected by FID, TV divergence, Recall, F1, and MAUVE scores, etc.) compared to \ftgenerator{}. 
{\textbf{(4)} Some methods consistently show a 32.0 accuracy for \openreviewrecom{} and 30.5 for \openreviewarea{} classification, due to the failure of the downstream \robertabase{} model under DP, always outputting majority class (see \cref{app:exp-details} for label distributions).}

\begin{figure*}[ht]
\vspace{-2mm}
\begin{minipage}{0.25\textwidth}
      \centering \includegraphics[width=1.05\linewidth]{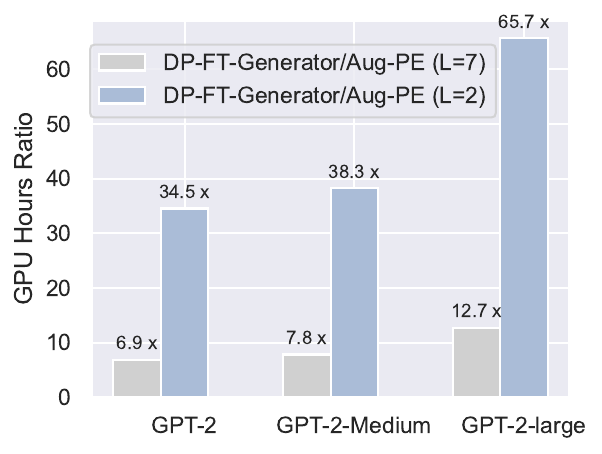} 
          \vspace{-8mm}
    \caption{\small Efficiency comparison between \ftgenerator{} and \name{} on \yelp{} for generating 100k synthetic samples ($\epsilon=1$) 
    }
    \label{fig:gpu-hour-yelp}
   \end{minipage}\hfill
\begin{minipage}{0.25\textwidth}
      \centering\includegraphics[width=1.03\linewidth]{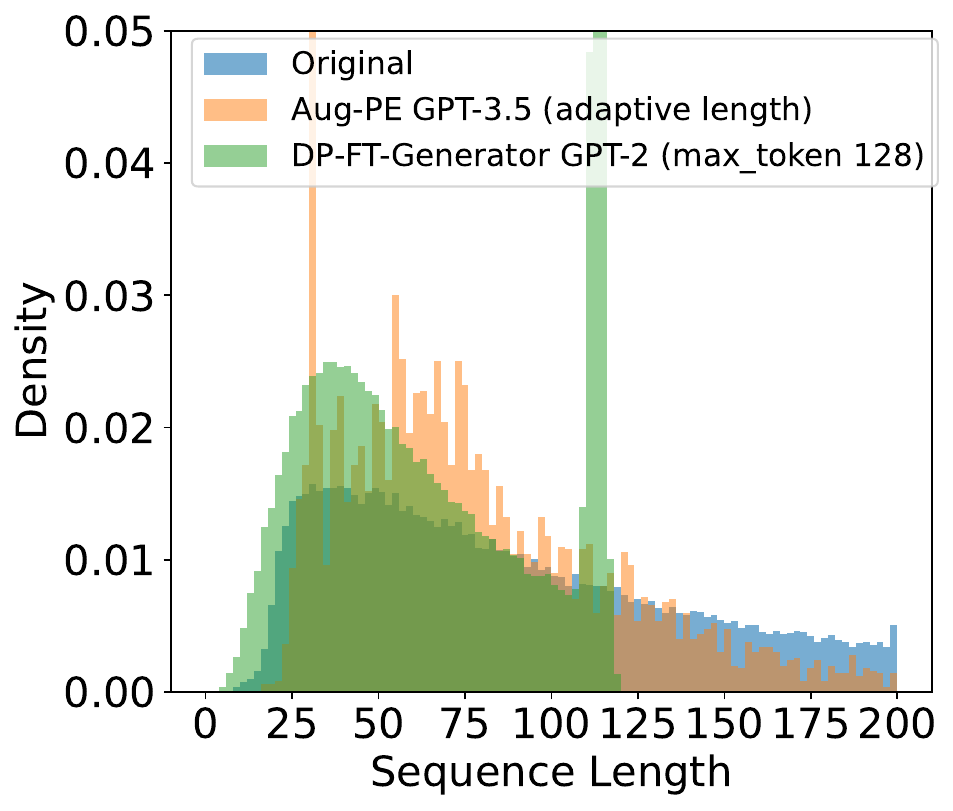}
     \vspace{-8mm}
   \caption{\small \chatgpt{} with adaptive text length achieves a comparable text length distribution to the original data on \yelp{}.
   }
   \label{fig:length-compare}
   \end{minipage}\hfill
   \begin{minipage}{0.23\textwidth}
     \centering\includegraphics[width=1.05\linewidth]{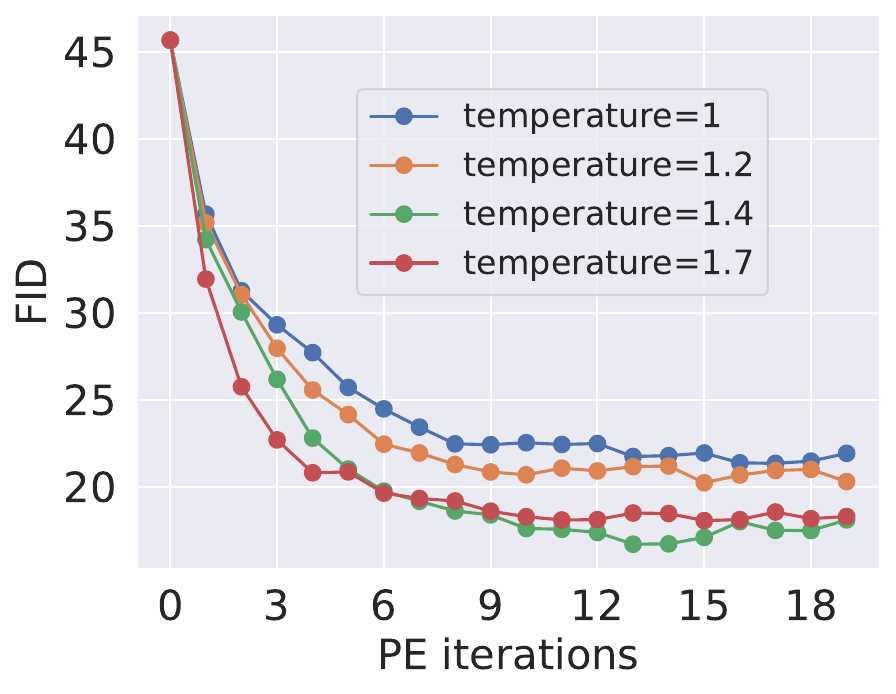}
     \vspace{-8mm}
    \caption{\small Larger temperature for  \gptthreepointfive{}  leads more diverse generation on \yelp{}  with a lower FID score.}
    \label{fig:yelp-gpt35-temp}
   \end{minipage}\hfill
   \begin{minipage}{0.23\textwidth}
    \centering\includegraphics[width=1.1\linewidth]{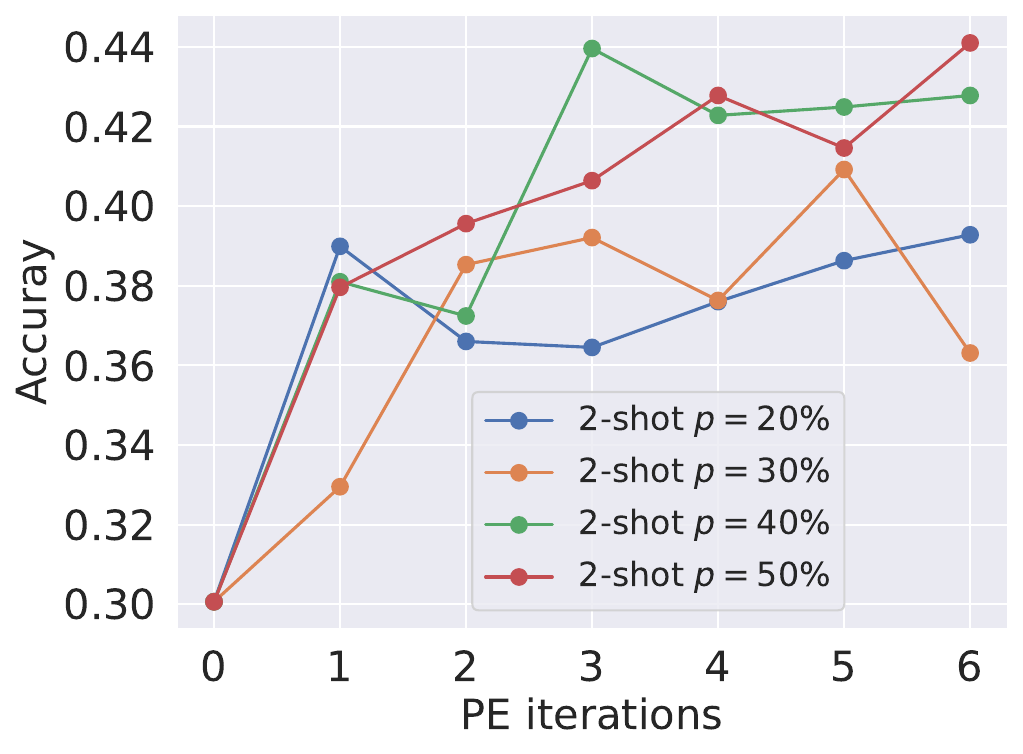}
    \vspace{-8mm}
    \caption{\small Fill-in-the-blanks with a larger mask probability $p\%$ for \gptthreepointfive{} leads to more diverse generation and higher utility on \openreview{}.}  %
    \label{fig:openreview_mask_prob}
   \end{minipage}
   \vspace{-6mm}
\end{figure*}

\textit{RQ2: Can DP synthetic texts from \name{}  be a better choice than \ftdownstream{} on real data with DP?}
\textbf{\name{} obtains comparable and higher accuracy than \ftdownstream{} under DP.}
\textbf{(1)} \cref{tb:main_result} shows that under $\epsilon=2,1$  on \pubmed, \name{} \gptthreepointfive{} with a smaller synthetic dataset size (2k) is sufficient to produce better downstream models compared to models directly trained with DP on the original data of the full (75k) or same size (2k). Similar conclusions hold for other two datasets, and the advantages of \name{} on \openreview{} are evident across all generators. 
\textbf{(2)} \ftdownstream{} performs fairly poor when the data size is small (e.g., 2k on \pubmed{} and \openreview{}), indicating that LMs finetuned with DP-SGD is unable to learn meaningful information under DP noises when samples are limited \citep{yu2021large,li2022large,bu2022differentially}.  In contrast, postprocessing property of DP allows us to train downstream tasks on DP synthetic text (with any size) via normal training techniques, without incurring additional privacy loss, potentially leading to a better downstream model than \ftdownstream{}.

\textit{RQ 3: How does \name{} perform across different privacy budget $\epsilon$?}
\textbf{(1)} \cref{tb:main_result} shows that \name{} in general achieves better performance as $\epsilon$ increases from $1,2,4$ to $\infty$, suggesting that \name{} scales well with the privacy budget $\epsilon$.
\textbf{(2)}
On \openreview{}, from $\epsilon=\infty \rightarrow 1$, the rating classification accuracy obtained from \ftgenerator{} \gpttwol{} generated text  drops  from $48.3\rightarrow 38.6$, and \ftdownstream{} on full training data drops from $65.1\rightarrow 30.5$, 
while the accuracy of \name{} \gptthreepointfive{} exhibits marginal drop  $45.4\rightarrow 43.1$. 
It suggests that in some cases, \textbf{the performance of \name{} (paired with powerful generator) can be more robust under DP noise than FT baselines}. 
The reason could be that LMs are vulnerable to the perturbations introduced in model parameters through DP-SGD, whereas \name{} strategically adds noise to the histogram votes, effectively preserving the utility. %

\begin{table*}[h]
\centering
\vspace{-2mm}
\caption{\small Using powerful LLMs as data generators leads to improved downstream accuracy on three datasets.}
\label{tab:effect_data_generator}
\vspace{-2mm}
\resizebox{0.8\textwidth}{!}{%
\begin{tabular}{l|cc|cc|cc|cc|cc|cc}
\toprule
& \multicolumn{4}{c|}{\yelp{}} & \multicolumn{4}{c|}{\openreview{}} & \multicolumn{4}{c}{\pubmed{}} \\
\cmidrule(lr){2-5}  \cmidrule(lr){6-9} \cmidrule(lr){10-13} 
& \multicolumn{2}{c}{\(\epsilon=\infty\)} & \multicolumn{2}{c|}{\(\epsilon=1\)} & \multicolumn{2}{c}{\(\epsilon=\infty\)} & \multicolumn{2}{c|}{\(\epsilon=1\)} & \multicolumn{2}{c}{\(\epsilon=\infty\)} & \multicolumn{2}{c}{\(\epsilon=1\)} \\
\cmidrule(lr){2-13} 
LLM & \yelprating & \yelpcategory & \yelprating & \yelpcategory & \openreviewarea & \openreviewrecom & \openreviewarea & \openreviewrecom & \bertmini & \bertsmall & \bertmini & \bertsmall \\
\midrule
\gpttwo & 67.5 & 74.8 & 66.9 & 74.4 & 42.4 & 32.1 & 37.6 & 32.0 & 24.5 & 26.7 & 24.3 & 26.5 \\
\gpttwom & 67.5 & 74.9 & 67.4 & 74.6 & 41.0 & 32.3 & 36.6 & 32.1 & 25.5 & 27.7 & 24.9 & 27.0 \\
\gpttwol & 67.5 & 74.5 & 66.6 & 75.0 & 42.1 & 32.1 & 38.1 & 32.0 & 25.7 & 27.9 & 25.1 & 27.2 \\
\midrule
Opt-6.7b & 68.7 & \textbf{75.3} & 67.7 & \textbf{75.3} & 43.6 & 32.2 & 30.5 & 32.1 & 26.5 & 28.6 & 25.8 & 27.9 \\
Vicuna-7b-v1.5 & \textbf{68.8} & 74.1 & 67.2 & 74.9 & 42.9 & 35.7 & 35.2 & 35.4 & 24.6 & 26.9 & 23.1 & 24.9 \\
Falcon-7b-instruct & 67.4 & 74.9 & 67.3 & 74.2 & 38.6 & 32.6 & 39.0 & 33.3 & 22.3 & 24.4 & 22.4 & 24.5 \\
Llama-2-7b-chat-hf & 68.6 & 74.9 & \textbf{68.0} & 75.1 & 45.5 & 38.5 & 36.4 & 37.0 & 25.8 & 28.4 & 24.8 & 27.5 \\
Mixtral-8x7B-v0.1 & 68.2 & 74.6 & 67.6 & 74.6 & \textbf{45.9} & 41.8 & \textbf{43.6} & 42.3 & 24.9 & 27.6 & 24.5 & 27.1 \\
\gptthreepointfive & 68.4 & 74.1 & 67.9  & 74.0 & 45.4 & \textbf{43.5} & 41.9 & \textbf{43.1} & \textbf{30.4} & \textbf{32.7} & \textbf{30.1} & \textbf{32.4} \\
\bottomrule
\end{tabular}
}
\vspace{-3mm}
\end{table*}

\textit{RQ 4:  Compared to \ftgenerator{}, how efficient the API-access-based \name{} is in terms of GPU hours? }
\textbf{With inference API access, \name{} is more efficient than \ftgenerator{} that requires DP-SGD finetuning}.   \textbf{(1)} As shown in \cref{fig:gpu-hour-yelp},  to generate 100k synthetic samples on \yelp{} under $\epsilon=1$, given the same generator  \gpttwol{}, \name{} $L=7$ provides 12.7x speedup and $L=2$ further provides 65.7x speedup.
\textbf{(2)}  The running time of \name{} is mainly scaled with \# API calls, which is associated with the number of variations $L-1$ in \cref{line:variationexpanded}.
\textbf{(3)} The bottleneck of \ftgenerator{} is DP-SGD finetuning:  it takes 1764 GPU hours on 32G NVIDIA V100 to finetune \gpttwol{} on \yelp{} and 7 hours to generate 100k samples, while \name{} $L=2$ ($L=7$) only requires 27 hours (139 hours). 
It highlights the computational expense of DP-SGD training, particularly for training LLMs, and underscores the efficiency of the API-based DP algorithm \name{}. A detailed breakdown of the GPU hours for each setting is in Appendix \cref{tab:gpu_hours}.
{\textbf{(4)} We use half precision (FP16) for LLM inference in \name{}. With the emerging efficient inference techniques (e.g., \citet{liu2023deja}), \name{} runtime can be further optimized.}

\textit{RQ 5: How robust \name{} is under empircal privacy attacks compared to  DP-finetuning-based baselines?} 
We perform state-of-the-art text membership inference attacks (MIAs) against the finetuned downstream models on \pubmed{} dataset. We consider three types of MIAs and report the AUC score: 
(1) \textit{PPL} thresholds perplexity to predict membership~\cite{carlini2021extracting}; 
(2) \textit{REFER} computes the ratio of the log perplexity of the tested model against a reference model~\cite{carlini2021extracting};
(3) \textit{LIRA} uses the ratio of likelihood~\cite{carlini2022membership} and we use the pre-trained model as a reference following~\cite{mattern2023membership}.
The results in \cref{tab:mia_auc} show that \textbf{\name{} generally exhibits lower AUC scores under MIAs} compared to \ftgenerator{} and \ftdownstream{}. This indicates a higher robustness to empirical privacy attacks, potentially due to the synthetic nature of the data used for downstream model finetuning, which inherently reduces the risk of overfitting to real private data. We defer the details to \cref{app:mia_attack}.

\begin{table}[ht]
\centering
\vspace{-2mm}
\caption{\name{} generally yields lower AUC scores against membership inference attacks on \pubmed{}  than \ftgenerator{} and \ftdownstream{}, indicating a higher robustness to empirical privacy attacks.}
\vspace{-2mm}
\label{tab:mia_auc}
\setlength{\tabcolsep}{3pt}
\resizebox{\columnwidth}{!}{%
\begin{tabular}{ l | c  |c c c|c c c|c c c|c }
\toprule
\textbf{Method} & \textbf{Generator} & \multicolumn{3}{c|}{\textbf{AUC} ($\epsilon=\infty$)} & \multicolumn{3}{c|}{\textbf{AUC} ($\epsilon=4$)} & \multicolumn{3}{c|}{\textbf{AUC} ($\epsilon=2$)} & \multirow{2}{*}{\textbf{Avg}} \\
\cmidrule{3-5} \cmidrule{6-8} \cmidrule{9-11}
& & PPL & REFER & LIRA & PPL & REFER & LIRA & PPL & REFER & LIRA  \\
\midrule
\ftdownstream{} & / & 77.60 & 74.93 & 65.05 & \textbf{49.32} & 54.58 & 56.82 & \textbf{48.96} & 50.41 & 51.56 & 58.80 \\ \midrule
\ftgenerator{} & {\gpttwo} & {55.50} & {51.35} & {51.97} & {53.90} & {51.12} & {51.84} & {53.31} & {50.81} & {51.61} & {52.38} \\
\ftgenerator{} & {GPT2-M} & {54.91} & {51.25} & {51.88} & {54.72} & {51.13} & {51.76} & {54.58} & {51.28} & {51.85} & {52.60} \\
\ftgenerator{} & {GPT2-L} & {54.81} & {51.22} & {51.86} & {54.56} & {50.81} & {51.64} & {55.05} & {51.01} & {51.69} & {52.52} \\\midrule
\name{} & \gpttwo & 50.08 & 50.92 & 51.66 & 50.10 & 50.97 & 51.70 & 49.94 & 50.85 & 51.64 & 50.87 \\
\name{} & GPT2-M & 49.85 & 50.73 & 51.57 & 50.10 & 50.95 & 51.69 & 49.73 & 50.65 & 51.51 & 50.75 \\
\name{} & GPT2-L & \textbf{49.43} & 50.40 & 51.40 & 49.61 & \textbf{50.56} & 51.48 & 49.66 & 50.60 & 51.49 & \textbf{50.51} \\
\name{} & \gptthreepointfive & 52.23 & \textbf{49.67} & \textbf{50.85} & 52.68 & 49.84 & \textbf{50.93} & 52.77 & \textbf{49.77} & \textbf{50.85} & 51.07 \\
\bottomrule
\end{tabular}
}
\vspace{-2mm}
\end{table}

\vspace{-2mm}
\subsection{Understanding the Properties of \name{}}
\label{sec:exp_understand_property}
Here we study properties of \name{} including text lengths, its compatibility with stronger data generators and downstream models, and its behaviors under data scaling. 

\textit{RQ 6: Can \name{} produce sentence length distributions similar to real data?}
\textbf{\name{} produces favorable text length distributions.}
From \cref{fig:length-compare}, we see that the text length distribution of synthetic samples produced from \chatgpt{} through \name{} is close to the distribution of the original \yelp{} data, highlighting the effectiveness of our adaptive sequence length mechanism (\cref{sec:pe_design}).  Note that the finetuning baseline requires a fixed $\mathrm{max\_token}$ (e.g., 128 for \gpttwo{}), which leads to a hard threshold for maximal text length, which is not the case in our method with our adaptive length technique. Nevertheless, there is a peak near 30 tokens for \name{}, which is due to the $\mathrm{\min\_word}$ set in the prompt to prevent empty generation. 
We defer the convergence of text length distributions over  PE iterations to \cref{app:sentence-len-converge}.

\textit{RQ 7: Can \name{}  benefit from more powerful LLMs?}
\textbf{\name{} is effective across a wide range of API-accessible LLMs.} 
We have observed from \cref{tb:main_result}  that \gptthreepointfive{} can lead to higher downstream accuracy than \gpttwo{}-series, especially on \pubmed{} and \openreview{}. 
Here we evaluate more API-accessible, non-GPT based LLMs.
\textbf{(1)} As shown in \cref{tab:effect_data_generator}, under $\epsilon=\infty, 1$, those modern LLMs 
can obtain comparable and even higher accuracy than \gptthreepointfive{} on \yelp{}, suggesting that \name{} can effectively elicit and select high-quality synthetic text from various types of LLMs. Note that \textit{DP finetuning often needs to be implemented case-by-case for LLMs and currently lacks open-source implementations for these LLMs}, whereas \name{} can easily leverage them.
{\textbf{(2)} The results on \openreview{} and \pubmed{} in \cref{tab:effect_data_generator} show that \gptthreepointfive{} leads to higher utility than opensource LLMs (e.g. \llamatwo{}), demonstrating the stronger generation power of \gptthreepointfive{} in academic/medical domains. Interestingly, Mixtral-8x7B can also generate high-quality synthetic texts for \openreview{}, but not for \pubmed{}.

\begin{table}[t]
\vspace{-1mm}
\caption{\small The next word prediction accuracy increases when using larger downstream models for \pubmed{} synthetic texts.
}
\vspace{-3mm}
\label{tab:effect_downstream_model}
\resizebox{\columnwidth}{!}{%
\begin{tabular}{l|l|l|ccccccc}\toprule
$\epsilon$ & Method & Generator  & bert-tiny   & bert-mini   & bert-small & Llama2-7b-chat-hf
\\ 
& &  &  4.4M & 11.2M & 28.8M & 7B \\
\midrule
 & \ftgenerator{} & \gpttwol{} & \textbf{24.6}  &  \textbf{31.0} & \textbf{33.1} & 53.1  \\
\rowcolor{lightgray}
 \cellcolor{white} \multirow{-2}{*}{$\infty$} &  \name{} & \gptthreepointfive  & 23.0  & 30.3 & 32.7 &  \textbf{56.5} \\
\midrule
 & \ftgenerator{} & \gpttwol{}   &  \textbf{23.1} & 28.9 &  31.1 & 52.0 \\
\rowcolor{lightgray}
\cellcolor{white} \multirow{-2}{*}{$1$}  & \name{} & \gptthreepointfive   & 22.9  & \textbf{30.1} &  \textbf{32.4} & \textbf{56.4} \\
  \bottomrule
\end{tabular}
}
\vspace{-4mm}
\end{table}

\textit{RQ 8: Can more powerful downstream models benefit from synthetic text generated via \name{}? }
\textbf{The high-quality synthetic text from \name{} is better utilized by larger downstream models.} 
\textbf{(1)} From each row in \cref{tab:effect_downstream_model}, we see that next-word prediction accuracy monotonically increases with the use of larger downstream models trained on \pubmed{} synthetic text.
\textbf{(2)} Under both $\epsilon=1,\infty$, the smallest model \berttiny{} favors the synthetic texts from \ftgenerator{} \gpttwol{}, while larger models such as  \llamatwo{} favor synthetic text from \name{} \gptthreepointfive{}. 
This observation underscores the importance of choosing downstream models of a suitable size; employing overly small models could under-estimate the quality of synthetic texts produced by \name{} with \gptthreepointfive{}.
{We hypothesize that this is because i)  
\gptthreepointfive{} generated texts might already be of higher quality in terms of vocabulary, syntax, semantic coherence, etc., compared to generated texts from finetuned \gpttwol{}; and ii) larger downstream LMs like \llamatwo{} can better understand and utilize the nuances in synthetic texts for improved performance than \berttiny{}.
}

\textit{RQ 9: Can we further improve downstream task accuracy with more synthetic samples generated from \name{}? }
To study the scaling law of \name{}, we use \gpttwo{}-series models to generate \{5k,10k,100k\} samples for \yelp{}, and \{2k,3k,5k\} samples for other two datasets. 
As shown in \cref{app:downstream_task_utility_full},  under $\epsilon=1,2,4,\infty$, \name{} in general achieves better performance across all datasets as the data size increases, suggesting that \textbf{\name{} scales well with the number of synthetic samples. }

\vspace{-1mm}
\subsection{Validating the Design of \name{}}
{As \name{} introduces novel sample selection and generation techniques, here we study algorithm components related to the two steps, respectively (under $\epsilon=\infty$), and compare its performance against the original \pename{}.}  

\textit{RQ 10: Can \name{} surpass original \pename{}? } {\cref{tab:gpt2_pe_augpe} shows that \textbf{\name{} achieves notable improvement over \pename{}} for \gpttwo{}, e.g., +22.6\% on \yelp{} rating classification. We observe similar conclusions for \gptthreepointfive{} in \cref{tab:pe_augpe_gpt3.5} in App.}
\begin{table}[h]
\vspace{-2mm}
\caption{\small \name{} outperforms \pename{} with \gpttwo{} on all datasets.}
\vspace{-3mm}
\label{tab:gpt2_pe_augpe}
\resizebox{\columnwidth}{!}{%
\begin{tabular}{l|ll|cc|cccc}\toprule
\multirow{2}{*}{Method}  &  \multicolumn{2}{c}{  \yelp{} }   &  \multicolumn{2}{|c|}{  \openreview{} }  &  \multicolumn{2}{c}{  \pubmed{} }    \\  
 &  \yelprating & \yelpcategory  &  \openreviewarea & \openreviewrecom  &  \bertmini & \bertsmall \\ \midrule
\pename{} $\gets$ \name{} ($k=6,L=1$) & 44.9 & 71.8  & 35.3 &  32.0  &  20.1 & 22.3 \\
\rowcolor{lightgray}
\name{} ($k=0,L=7$)  & \textbf{67.5} & \textbf{74.8} &  \textbf{42.4}  &  \textbf{32.1}  & \textbf{24.5} & \textbf{26.7}   \\
 \bottomrule
\end{tabular}
}
\vspace{-2mm}
\end{table}

\textit{RQ 11: How does the private data guided sample selection affect \name{} performance? } Here we aim to verify the \textbf{components related to sample selection: i) usage of private data; ii) rank-based selection; iii) embedding model} used during nearest neighbor voting.

\noindent\textbf{i) Usage of private data.} \cref{tab:effect_private_data} shows that the initial samples (generated from Random API) or their variants (generated from Random API + Variation API) exhibit limited utility without using private data. 
However, the quality of the synthetic text improves notably after just one iteration of \name{} ($t=1$) when guided by private data, and this improvement continues to amplify with $T$ iterations. We report the results under DP in \cref{app:effect_t_dp}.
 
\begin{table}[h]
\vspace{-3mm}
\caption{\small Private-data guided sample selection in \name{} improves the utility of \gptthreepointfive{} generated texts. 
}
\vspace{-2mm}
\label{tab:effect_private_data}
\resizebox{\columnwidth}{!}{%
\begin{tabular}{l|ll|ll|cccccccc}\toprule
Setting &   \multicolumn{2}{c}{ \yelp{}}  &   \multicolumn{2}{|c|}{ \openreview{}}   &   \multicolumn{2}{c}{ \pubmed{}}      \\ 
 &    \yelprating & \yelpcategory   &    \openreviewarea & \openreviewrecom  & \bertmini & \bertsmall \\\midrule
Random API & 62.3  &  73.7 &  34.4  & 42.0 & 29.7  & 31.9 \\
Random API + Variation API &62.3  &  73.7 &  36.4  & 42.0 & 29.6 & 31.9  \\
\rowcolor{lightgray}
\name{} ($t=1$) &  64.9 & 73.8 & 39.3  & 42.5 & 30.0  & 32.2 \\
\rowcolor{lightgray}
\name{} ($t=T$)  & \textbf{68.4} & \textbf{74.1}  & \textbf{45.4}&  \textbf{43.5} & \textbf{30.4} & \textbf{32.7} \\
 \bottomrule
\end{tabular}
}
\vspace{-2mm}
\end{table}

\noindent\textbf{ii) Rank-based sampling.} 
The results in \cref{app:exp_rank_ablation} indicate that our proposed rank-based sampling (\cref{line:directuse}) consistently outperforms probability-based random sampling in the original \pename{} (\cref{line:drawfromhist}), due to the elimination of sample redundancy inherent in random sampling, as rank-based sampling exclusively selects the top $N_{\syn}$ samples.

\textbf{iii) Embedding models.} 
\cref{tab:effect_embedding_model} shows that larger embedding models such as ``sentence-t5-xl'' can more accurately capture the nuances of texts in the embedding space, leading to higher utility for \gpttwo{} generated texts.

\begin{table}[h]
\vspace{-2mm}
\caption{\small More powerful embedding model leads to higher utility for \gpttwo{} generated texts via \name{}.}
\vspace{-3mm}
\label{tab:effect_embedding_model}
\resizebox{\columnwidth}{!}{%
\begin{tabular}{l|ll|ccccccc}\toprule
 \multirow{2}{*}{\shortstack{Embeddding model\\ \cite{reimers2019sentence}}} & \multicolumn{2}{c|}{ \yelp{}}  &    \multicolumn{2}{c}{ \pubmed{}}    \\ 
 &    Rating & Category   &    \bertmini & \bertsmall  \\\midrule
sentence-t5-xl & \textbf{67.6} & 75.1  & \textbf{25.1} & \textbf{27.4} \\    
sentence-t5-base & 67.2 & 75.2 &   24.5 & 26.7 \\          
stsb-roberta-base-v2 & 67.5 & 74.8 & 23.9 & 26.1 \\
all-MiniLM-L6-v2 & 62.6 & \textbf{75.3} &   24.7 & 26.7  \\   
paraphrase-MiniLM-L6-v2 & 64.7 & 75.1   & 24.3 & 26.5 \\
all-mpnet-base-v2 & 64.1 & 74.6   & 24.0 & 26.0   \\
 \bottomrule
\end{tabular}
}
\vspace{-3mm}
\end{table}

\textit{RQ 11: How to improve the generation quality through Variation API in \name{}? }
We analyze key \textbf{componenents related to generation: i) variation API prompt designs; ii) LLMs generation configuration (e.g., temperature); iii) number of variations $L-1$}. 

\vspace{-1mm}
\textbf{i) Variation API prompt designs}. We evaluate the impact of four types of Variation API prompts on \yelp{}: paraphrasing and fill-in-the-blanks prompts under zero-shot and few-shot settings. 
\textbf{(1)} Qualitatively, we observed that \gpttwo{} struggles to adhere to the fill-in-the-blanks instruction, often leaving blanks (``$\_\_$'') in the generated texts. In contrast, \gptthreepointfive{} can effectively fill in the blanks, potentially because \gptthreepointfive{}  has been instruction-tuned~\cite{wei2021finetuned} and thus follows the instructions better.
\textbf{(2)} The quantitative results in {Appendix} \cref{tab:effect_variation_api} reveal that paraphrasing  can be an effective strategy for \gpttwo{}, while fill-in-the-blanks yields better results for \gptthreepointfive{}.
\textbf{(3)} Fill-in-the-blanks offers more control over the diversity of generated content. By increasing the mask probability $p\%$, we can create more room for imaginative responses from \gptthreepointfive{}, leading to more diverse generations. As indicated in \cref{fig:openreview_mask_prob}, a higher mask probability corresponds to increased accuracy in downstream area classification tasks when using \gptthreepointfive{}.

\begin{table}[h]
\vspace{-2mm}
\caption{\small 
For \gpttwo{} generated texts, high temperatures are preferred for \yelp{} while moderate temperatures are favored for \openreview{} and \pubmed{} to balance generation diversity and quality.}
\vspace{-3mm}
\label{tab:effect_temperature}
\resizebox{\columnwidth}{!}{%
\begin{tabular}{l|ll|ll|ccccc}\toprule
\multirow{2}{*}{Temperature}  &  \multicolumn{2}{c}{  \yelp{} }   &  \multicolumn{2}{|c|}{  \openreview{} }  &  \multicolumn{2}{c}{  \pubmed{} }    \\  
 &    \yelprating & \yelpcategory  &    \openreviewarea & \openreviewrecom  &  \bertmini & \bertsmall \\ \midrule
0.8 & 66.9 & 74.2 & 42.0 & \textbf{32.2} &  24.5 & \textbf{26.8} \\
1.0 & 66.8 & 74.8  & 41.5 & 32.1 & \textbf{24.5} & 26.7  \\
1.2 & 67.0 & 74.9 & \textbf{42.4} & 32.1 & 24.4& 26.5 \\
1.4 & \textbf{67.5} & 74.8 & 40.8 & 32.0 & 23.6 & 25.6 \\
1.7 & 67.1 & \textbf{75.2} & 40.6 & 32.1 & 21.9 & 24.0 \\
 \bottomrule
\end{tabular}
}
\vspace{-5mm}
\end{table}

\textbf{ii) Temperature} is a key parameter in controlling the diversity of LLM generation. A higher temperature leads LLMs to generate less frequent tokens, thereby increasing diversity. However, an excessively high temperature may result in overly random outputs and potentially hurt generation. The impact of different temperatures for \name{} on \gpttwo{} is shown in \cref{tab:effect_temperature}. 
\textbf{(1)} On \yelp{}, a higher temperature (1.4 to 1.7) proves beneficial for  \gpttwo{}, as business reviews often encompass daily conversations with a variety of sentence formats and tones. Additional findings in \cref{fig:yelp-gpt35-temp} indicate that large temperatures can also lead to low (better) FID scores for \gptthreepointfive{}.
\textbf{(2)}  Conversely, on \openreview{} and \pubmed{}, a moderate temperature setting (around 1.0) is more suitable for \gpttwo{}, as academic and medical literature demand more precise and accurate text generation.

\begin{table}[h]
\vspace{-3mm}
\caption{\small Increasing the number of variations $L-1$ in \name{} yields higher utility for  \gpttwo{} generated texts.}
\vspace{-3mm}
\label{tab:effect_num_variations}
\resizebox{\columnwidth}{!}{%
\begin{tabular}{l|cc|cc|cccc}\toprule
\multirow{2}{*}{$L-1$}   &  \multicolumn{2}{c}{  \yelp{} }   &  \multicolumn{2}{|c|}{  \openreview{} }  &  \multicolumn{2}{c}{  \pubmed{} }    \\  
 &    \yelprating & \yelpcategory  &    \openreviewarea & \openreviewrecom  &  \bertmini & \bertsmall \\ \midrule
1 & 65.8 & 74.4 & 39.2 & \textbf{32.1} & 23.9 & 26.1 \\
3 & 66.7 & \textbf{75.1} & 41.1 & 32.0 & 24.6 & 26.8\\
6 & 67.5 & 74.8 & 42.4 & \textbf{32.1} & 24.5 & 26.7 \\
9 & \textbf{67.7} & 74.9 & \textbf{42.7} & 32.0 & \textbf{24.9} & \textbf{26.8} \\
 \bottomrule
\end{tabular}
}
\vspace{-2mm}
\end{table}

\textbf{iii) Increasing the number of variations $L-1$} generally enhances  performance of \name{} as shown in \cref{tab:effect_num_variations}, due to the expansion of the candidate synthetic sample pool, which increases the likelihood of getting high-quality texts.
However, generating more variations requires additional API calls, leading to increased computational costs as discussed in \cref{fig:gpu-hour-yelp}. To  balance the trade-off between utility and efficiency, we use $L=7$ for \gpttwo{}-series experiments.

\myparatightestn{\name{} convergence.}
We provide generation results showing the convergence of \name{} under \textit{one private sample} in \cref{app:converge_one_private_sample}, which demonstrate our sample selection and generation process in a more direct manner.

\vspace{-2mm}
\section{Conclusion}
\vspace{-1mm}
In this work, we propose \name{} for DP synthetic text generation without model training. We conduct comprehensive experiments on three datasets and show that \name{} can generate high-quality DP synthetic text with comparable privacy-utility tradeoff to DP finetuning baselines under the same data generator. Leveraging more powerful open-source LLMs or API-based LLMs as data generators,  \name{} can generate DP synthetic text with improved utility.  

\section*{Acknowledgements}
The authors thank Xiang Yue, Janardhan Kulkarni and anonymous reviewers for their valuable feedback and suggestions.
BL gratefully acknowledges the support from the National Science Foundation under grants No. 1910100, No. 2046726, No. 2229876, and the Alfred P. Sloan Fellowship.

\bibliography{ref}
\bibliographystyle{icml2024}
\newpage
\appendix

\onecolumn

\section{Privacy Analysis}
\label{app:privacy-analysis}

We first introduce a related theorem from \citet{balle2018improving} in \cref{thm:anlytic_gaussian}.
\begin{theorem}[Analytic Gaussian Mechanism~\cite{balle2018improving}]
\label{thm:anlytic_gaussian}
    Let $f: \mathbb{X} \rightarrow \mathbb{R}^d$ be a function with global $L_2$ sensitivity $\Delta$. For any $\varepsilon \geq 0$ and $\delta \in[0,1]$, the Gaussian output perturbation mechanism $M(x)=f(x)+Z$ with $Z \sim \mathcal{N}\left(0, \sigma^2 I\right)$ is $(\varepsilon, \delta)-D P$ if and only if
$$
\Phi\left(\frac{\Delta}{2 \sigma}-\frac{\varepsilon \sigma}{\Delta}\right)-e^{\varepsilon} \Phi\left(-\frac{\Delta}{2 \sigma}-\frac{\varepsilon \sigma}{\Delta}\right) \leq \delta .
$$
\end{theorem}

Next, we provide the privacy guarantee for \cref{algo} in \cref{thm:our_privacy}
\begin{theorem}[Privacy Guarantee for \cref{algo}]
\label{thm:our_privacy}
    Let \cref{algo} run $T$ iterations, with noise multiplier $\sigma$ (noise is added to each bin of the histogram), the DP mechanism satisfies $(\varepsilon, \delta)$-DP if and only if
$$
\Phi\left(\frac{\sqrt{T}}{2 \sigma}-\frac{\varepsilon \sigma}{\sqrt{T}}\right)-e^{\varepsilon} \Phi\left(-\frac{\sqrt{T}}{2 \sigma}-\frac{\varepsilon \sigma}{\sqrt{T}}\right) \leq \delta .
$$
\end{theorem}

\begin{proof}[Proof Sketch]
The proof is very similar to the one in \citet{lin2023differentially}. So we just describe the key steps at a high level. The $L_2$ sensitivity of the histogram created in each iteration of \cref{algo} is $\Delta=1$, to which we add Gaussian noise of scale $\sigma$. Therefore $T$ iterations of the algorithm can be seen as the adaptive composition of $T$ Gaussian mechanisms with $L_2$ sensitivity $1$ and noise scale $\sigma.$ The privacy loss of the composition is equivalent to that of a single Gaussian mechanism with $L_2$ sensitivity $1$ and noise scale $\sigma/\sqrt{T}$ according to the adaptive composition theorem of Gaussian mechanisms (Corollary 3. of ~\cite{dong2019gaussian}). Therefore the privacy gaurantee follows from Theorem~\ref{thm:anlytic_gaussian}.
\end{proof}

\section{Additional Experimental Details}
\label{app:exp-details}
\subsection{Datasets and Downstream Tasks.} 
\begin{table*}[ht]
\vspace{-3mm}
\renewcommand{\arraystretch}{1.1}
\centering 
\small
\caption{Dataset details.} 
\vspace{-2mm}
\label{tb:detaset-details}
\resizebox{\linewidth}{!}{
\begin{tabular}{cccccccccccc}
\toprule
  Dataset  &  \# Train & \#  Val & \#  Test & label 1 & label 2  \\
\midrule
\yelp{}   & 1.9M & 5000 & 5000  &  business category (10 classes)  & review ratings (5 classes)  \\
\openreview{} (ICLR2023)  & 8396 & 2798 &  2798  & review area  (12 classes)  &  review recommendation (5 classes)  \\
\pubmed{} (2023/08/01-2023/08/07) & 75316 & 14423 & 4453 &  \multicolumn{2}{c}{next-word prediction} \\
\bottomrule
\end{tabular}
}
\end{table*}
We evaluate \name{} on there datasets: 
\begin{itemize}[noitemsep,leftmargin=*]
    \item \yelp{}: \yelp{} data is a public benchmark providing reviews on businesses, and we used the preprocessed \yelp{} from \cite{yue2022synthetic}. The number of train/val/test samples and label information in \cref{tb:detaset-details}.
    \item  \openreview{}: For \openreview{} ICLR2023 data, we crawl the meta-data for each review using the OpenReview Python library,\footnote{https://github.com/openreview/openreview-py} and concatenate the fields ``summary\_of\_the\_paper'',  ``strength\_and\_weaknesses'' and ``summary\_of\_the\_review'' as one sample in our dataset. 
We group the two attributes -- review area and recommendation -- together as a combination, and drop the training samples from combinations that contain fewer than 50 training samples. The number of samples after such preprocessing and label information is provided in \cref{tb:detaset-details}. The number of samples for each class is provided in  \cref{tab:openreview_label_area} and \cref{tab:openreview_label_recom}. 
    \item \pubmed{}: we use PubMed with abstracts of medical papers\footnote{https://www.ncbi.nlm.nih.gov/} crawled by \citet{yu2023training} from 2023/08/01 to 2023/08/07. The number of train/val/test samples are reported in \cref{tb:detaset-details}.
\end{itemize}

For \yelp{} and \openreview{}, we focus on conditional generation and use two attributes (i.e., labels) for each dataset: the review ratings (ranging from 1 star to 5 stars) and business category for \yelp{} data, and the review recommendation (ranging from  ``1: strong reject'' to ``8: accept, good paper'') and review area for \openreview{} ICLR2023 data. 
We then use those labels for downstream classification tasks based on synthetic texts. 

For \pubmed{}, we focus on unconditional generation and use next-word prediction as downstream tasks.  This is motivated by \cite{yu2023training}

\begin{table}[h]
\centering
\caption{\openreviewarea{} label statistics of \openreview{}.}
\vspace{-2mm}
\label{tab:openreview_label_area}
\resizebox{\linewidth}{!}{
\begin{tabular}{lcc}
\toprule
Class Name  &  \# Train Samples (Proportion) & \# Test Samples (Proportion) \\ \midrule
Deep Learning and Representational Learning & 2479 (29.53\%) & 854 (30.52\%) \\
Applications (e.g., speech processing, computer vision, NLP) & 1100 (13.10\%)  & 380 (13.58\%) \\
Reinforcement Learning (e.g., decision and control, planning, hierarchical RL, robotics) & 1016 (12.10\%)  & 344 (12.29\%) \\
Social Aspects of Machine Learning (e.g., AI safety, fairness, privacy, interpretability, human-AI interaction, ethics) & 765 (9.11\%)  & 248 (8.86\%) \\
General Machine Learning & 598 (7.12\%)  & 177 (6.33\%) \\
Theory (e.g., control theory, learning theory, algorithmic game theory) & 458 (5.45\%)  & 144 (5.15\%) \\
Unsupervised and Self-supervised Learning & 452 (5.38\%)   & 135 (4.82\%) \\
Machine Learning for Sciences (e.g., biology, physics, health sciences, social sciences, climate/sustainability) & 440 (5.24\%)   & 166 (5.93\%) \\
Generative Models & 390 (4.65\%)  & 119 (4.25\%) \\
Optimization (e.g., convex and non-convex optimization) & 318 (3.79\%)   & 96 (3.43\%) \\
Probabilistic Methods (e.g., variational inference, causal inference, Gaussian processes) & 230 (2.74\%)  & 81 (2.89\%) \\
Neuroscience and Cognitive Science (e.g., neural coding, brain-computer interfaces) & 150 (1.79\%)   & 54 (1.93\%) \\
\bottomrule
\end{tabular}
}
\end{table}

\begin{table}[h]
\centering
\vspace{-2mm}
\caption{\openreviewrecom{} label statistics of \openreview{}.}
\vspace{-2mm}
\label{tab:openreview_label_recom}
\resizebox{0.9\linewidth}{!}{
\begin{tabular}{lcc}
\toprule
Class Name                                        &  \# Train Samples (Proportion) & \# Test Samples (Proportion) \\ \midrule
Recommendation: 6: marginally above the acceptance threshold & 2870  (34.18\%) & 896  (32.02\%)  \\
Recommendation: 5: marginally below the acceptance threshold & 2144 (25.54\%) &   760 (27.16\%)  \\
Recommendation: 3: reject, not good enough & 1703 (20.28\% ) &571  (20.41\%) \\
Recommendation: 8: accept, good paper & 1629 (19.40\%) &554  (19.80\%) \\
Recommendation: 1: strong reject & 50  (0.60\%) &17  (0.61\%) \\
\bottomrule
\end{tabular}
}
\end{table}

\subsection{Implementation Details of \name{}.}
\subsubsection{Model and Hyperparameters}
We consider four LLMs as data generators in \name{} via API-access: \gpttwo{}~\cite{radford2019language}, \gpttwom{}, \gpttwol{}, and \gptthreepointfive{} (``gpt-35-turbo'' hosted on Microsft Azure\footnote{https://learn.microsoft.com/en-us/azure/ai-services/openai/concepts/models})~\cite{chatgpt}. We provide the default hyper-parameter setup for GPT-3.5 in \cref{tb:hyper-gpt35} and GPT-2 series models in \cref{tb:hyper-gpt2}.

The embedding model $\embeddingnetworkname{}$  in \name{} is instantiated by the sentence-transformer from HuggingFace. We use ``stsbroberta-base-v2'' for \openreview{} and \yelp{} and ``sentence-t5-base'' for \pubmed{}.

After generating the synthetic samples, we remove those with fewer than 100/50 tokens for \openreview{}/\pubmed{}. We noticed that samples with token lengths below those thresholds usually result from an unsuccessful API call for paper review/medical abstract generation  (e.g. \gptthreepointfive{} refuses to answer). 

In terms of downstream models,
\begin{itemize}
    \item For \yelp{} and \openreview{},  we finetune the pre-trained \robertabase{} model for all downstream text classification tasks. We set the max sequence length as 512, the batch size as 64, the learning rate as 3e-5, and the number of epochs as 5 for \yelp{} and 10 for \openreview{}. 
    \item For \pubmed{}, we leverage pre-trained \bertmini{} and \bertsmall{} released by \cite{turc2019well}, which are lightweight to meet the inference time and computational cost requirements in many real-use cases. These models employ WordPiece tokenization and were trained on Wikipedia and BookCorpus using masked language modeling. During our downstream task fine-tuning, we implement a causal language modeling mask, restricting each token to attend only to its preceding tokens~\cite{yu2023training}.
We set the max sequence length as 512, batch size as 32, learning rate as 3e-4, the weight decay as 0.01. We finetune 20 epochs for \bertmini{} and 10 for \bertsmall{} epochs.
\end{itemize}

\begin{table*}[ht]
\vspace{-2mm}
\renewcommand{\arraystretch}{1.1}
\centering 
\small
\caption{Hyperparameters for GPT-3.5.} 
\vspace{-2mm}
\label{tb:hyper-gpt35}
\resizebox{\linewidth}{!}{
\begin{tabular}{cccccccccccc}
\toprule
    & $N_\mathrm{syn}$  & $K$ &  \samplevariationapiname{} & mask prob. p\%   &  $L$ &   PE iteration & temperature &  $\mathrm{w2t\_ratio}$    & $\sigma_{word}$  & $\mathrm{\min\_word}$  &  $\mathrm{\max\_token}$ for  \randomsampleapiname{}  \\
\midrule
 \yelp{}  & 5k  &  3   & fill-in-the-blanks (3-shot)& 50\% & 1  &  20 & 1.4 & 1.2 & 40  &25   & 128  \\
  \openreview{}  & 2k  &   0   & fill-in-the-blanks (1-shot)& 50\% & 4  &  10 & 1.2 & 5 & 60  & 25   & 1000  \\
  \pubmed{}  & 2k  &   0   & fill-in-the-blanks (0-shot)& 50\% & 4  &  10 & 1.2 & 5 & 60  & 25   & 1000  \\
\bottomrule
\end{tabular}
}
\end{table*}

\begin{table*}[ht]
\renewcommand{\arraystretch}{1.1}
\centering 
\small
\caption{Hyperparameters for GPT-2, GPT-2-Medium, and GPT-2-Large.} 
\vspace{-2mm}
\label{tb:hyper-gpt2}
\resizebox{0.9\linewidth}{!}{
\begin{tabular}{cccccccccccc}
\toprule
  Model  & $N_\mathrm{syn}$  & $K$ &  \samplevariationapiname{}    &  $L$ &   PE iteration $T$ & temperature &  $\mathrm{\max\_token}$      &  \\
\midrule
\yelp{}   & 5k, 10k, 100k  &   0  & paraphrasing (zero-shot) &  7  &  20 & 1.4 & 64     \\
\openreview{}   & 2k, 3k, 5k  &   0  & paraphrasing (zero-shot) &  7  &  10 & 1.2 & 448     \\
\pubmed{}   & 2k, 3k, 5k  &   0  & paraphrasing (zero-shot) &  7  &  10 & 1.0 & 448     \\
\bottomrule
\end{tabular}
}
\end{table*}

\begin{table}[h]\small
\centering
\vspace{-2mm}
\caption{\small Prompts as \randomsampleapiname{} for \gptthreepointfive{}.
}
\label{tab:random_api}
\begin{tabular}{>{\raggedright\arraybackslash}p{1cm} >{\raggedright\arraybackslash}p{4cm} >{\raggedright\arraybackslash}p{5cm} >{\raggedright\arraybackslash}p{3cm}
}
\toprule
\textbf{Speaker} & \textbf{\yelp{}} & \textbf{\openreview{}} & \textbf{\pubmed{}} \\
\midrule
System 
& You are required to write an example of review based on the provided Business Category and Review Stars that fall within the range of 1.0-5.0. 
& Given the area and final decision of a research paper, you are required to provide an example of the review consisting of the following content: 1. briefly summarizing the paper in 3-5 sentences; 2. listing the strengths and weaknesses of the paper in details; 3. briefly summarizing the review in 3-5 sentences. 
& Please act as a sentence generator for the medical domain. Generated sentences should mimic the style of PubMed journal articles, using a variety of sentence structures. \\
\midrule
User & Business Category: \{$\mathrm{label\_1}$\} | Review Stars: \{$\mathrm{label\_2}$\}  with  keyword \{$\mathrm{subcategory}$\} 
&    Area: \{$\mathrm{label\_1}$\}	| Recommendation: \{$\mathrm{label\_2}$\} 
&  Suppose that you are a \{$\mathrm{writer}$\}. Please write an abstract for a medical research paper:
\\
\bottomrule
\end{tabular}
\end{table}

\begin{table}[h]\small
\centering
\vspace{-2mm}
\caption{\small Prompts as \randomsampleapiname{} for \gpttwo{}-series models. 
}
\label{tab:random_api_gpt2}
\begin{tabular}{>{\raggedright\arraybackslash}p{4cm} >{\raggedright\arraybackslash}p{5cm} >{\raggedright\arraybackslash}p{3cm}
}
\toprule
 \textbf{\yelp{}} & \textbf{\openreview{}} & \textbf{\pubmed{}} \\
\midrule
 Business Category: \{$\mathrm{label\_1}$\} | Review Stars: \{$\mathrm{label\_2}$\}  with  keyword \{$\mathrm{subcategory}$\} 
&   Suppose that you are a \{$\mathrm{writer}$\}. Write a paper review based on  Area: \{$\mathrm{label\_1}$\}	| Recommendation: \{$\mathrm{label\_2}$\} 
&  Using a variety of sentence structures, write an abstract for a medical research paper:
\\
\bottomrule
\end{tabular}
\end{table}

\begin{table}[h]\small
\centering
\caption{\small Prompts as \samplevariationapiname{} for \gpttwo{}-series models on \yelp{} and \openreview{}.
}
\label{tab:var_api_gpt2}
\begin{tabular}{>{\raggedright\arraybackslash}p{2cm} >{\raggedright\arraybackslash}p{14cm}}
\toprule
\textbf{Datast} & \textbf{Prompt}  \\
\midrule
\yelp{}& Based on ``Business Category: \{$\mathrm{label\_1}$\} | Review Stars: \{$\mathrm{label\_2}$\}'', please rephrase the following sentences \{in a $\mathrm{selected\_tone}$\}:\\
&\{$\mathrm{input}$\}
\\
\midrule
\openreview{} & Based on ``Area: \{$\mathrm{label\_1}$\} | Recommendation: \{$\mathrm{label\_2}$\}'', please rephrase the following sentences \{in a $\mathrm{selected\_tone}$\}:\\
&\{$\mathrm{input}$\}
\\\midrule
\pubmed{} & Please rephrase the following sentences \{in a $\mathrm{selected\_tone}$\} as an abstract for medical research paper: \\
&\{$\mathrm{input}$\}\\
\bottomrule
\end{tabular}
\end{table}

\begin{table}[htp!]\small
\centering
\caption{\small Prompts as \samplevariationapiname{} for  GPT-3.5 on \yelp{}.
}
\label{tab:var_api_gpt35_yelp}
\resizebox{\linewidth}{!}{
\begin{tabular}{>{\raggedright\arraybackslash}p{1cm} >{\raggedright\arraybackslash}p{14cm}}
\toprule
\textbf{Speaker} & \textbf{Prompt}  \\
\midrule
System &You are a helpful, pattern-following assistant. \\
\midrule
User & Based on the Business Category and Review Stars, you are required to fill in the blanks in the Input sentences. If there are no blanks, you are required to output the original Input sentences. \\
&\\
& Business Category: Restaurants | Review Stars: 2.0\\
&Input: \_ that great , terrible \_ rolls and fish \_ smelling \_ \_.\\
&Fill-in-Blanks and your answer MUST be exactly 10 words: Not that great, terrible egg rolls and fishy smelling shrimp.\\
&\\
& Business Category: Beauty \& Spas | Review Stars: 5.0 \\
& Input: Very clean! Staff are super friendly!!\\
& Fill-in-Blanks and your answer MUST be exactly 6 words: Very clean! Staff are super friendly!!\\
&\\
& Business Category: Shopping | Review Stars: 3.0 \\
& Input: I \_ in \_ and stopped in for a \_. I was \_ surprised. Good \_, nice price.\\
&Fill-in-Blanks and your answer MUST be exactly 19 words: I was in a rush and stopped in for a mani-pedi. I was pleasantly surprised. Good service, nice price.\\
&\\
& Business Category: \{$\mathrm{label\_1}$\} | Review Stars: \{$\mathrm{label\_2}$\}  \\
& Input: \{$\mathrm{masked\_input}$\} \\
& Fill-in-Blanks and your answer MUST be exactly $\{\mathrm{targeted\_word}\}$ words:
\\
\bottomrule
\end{tabular}
}
\end{table}

\begin{table}[h]
\small
\centering
\caption{\small Prompts as \samplevariationapiname{} for  GPT-3.5 on \openreview{}. 
}
\label{tab:var_api_gpt35_openreview}
\resizebox{\linewidth}{!}{
\begin{tabular}{>{\raggedright\arraybackslash}p{1cm} >{\raggedright\arraybackslash}p{15.5cm}}
\toprule
\textbf{Speaker} & \textbf{Prompt}  \\
\midrule
System & You are an AI assistant that helps people find information. \\
\midrule
User & Based on the area and final recommendation of a research paper, you are required to fill in the blanks for the input sentences \{in a $\mathrm{selected\_tone}$\}. If there is no blanks, please output the original input sentences.\\
&\\
& Area: Applications (eg, speech processing, computer vision, NLP) | Recommendation: 3: reject, not good enough \\
& Input: \_\_ proposes an\_\_ method\_ ROI detection\_\_arial\_f\_ without attention\_. The\_ map can\_ used\_\_\_\_ for\_\_ and\_\_\_\_ show\_ improvements on different medical\_\_.\_Strength\_\_ $\backslash$n--The idea using\_\_actual images\_ sali\_\_ generation\_ interesting.$\backslash$n$\backslash$n\_The improvement\_\_\_\_aks is significant. $\backslash$n$\backslash$nWeak\_\_\_\_The\_\_\_ and\_\_\_\_\_ experiments are needed\_ such as\_\_f\_\_\_the\_ method\_ interesting\_ but\_ novelty\_ limited \\
&Fill-in-Blanks and your answer MUST be exactly 85 words: This paper proposes an attention generation method for ROI detection by adversarial counterfactual without attention label. The attention map can be used to highlight useful information for disease classification and detection. The experiments show its improvements on different medical imaging tasks.  $\backslash$nStrengths: $\backslash$n--The idea using counterfactual images for saliency map generation is interesting.$\backslash$n$\backslash$n--The improvement for medical imaging taks is significant. $\backslash$n$\backslash$nWeaknesses:$\backslash$n$\backslash$n--The novelty is simple and limited. $\backslash$n$\backslash$n--More experiments are needed, such as existing counterfactual generation.$\backslash$nthe proposed method is interesting, but the novelty is limited.\\
&\\
&\\
& Area: \{$\mathrm{label\_1}$\} | Recommendation: \{$\mathrm{label\_2}$\}  \\
& Input: \{$\mathrm{masked\_input}$\} \\
& Fill-in-Blanks and your answer MUST be exactly $\{\mathrm{targeted\_word}\}$ words:
\\
\bottomrule
\end{tabular}
}
\end{table}

\begin{table}[h]
\small
\centering
\caption{\small Prompts as \samplevariationapiname{} for  GPT-3.5 on \pubmed{}. 
}
\label{tab:var_api_gpt35_pubmed}
\resizebox{\linewidth}{!}{
\begin{tabular}{>{\raggedright\arraybackslash}p{1cm} >{\raggedright\arraybackslash}p{15.5cm}}
\toprule
\textbf{Speaker} & \textbf{Prompt}  \\
\midrule
System & Please act as a sentence generator for the medical domain. Generated sentences should mimic the style of PubMed journal articles, using a variety of sentence structures.\\
\midrule
User & 
You are required to fill in the blanks with more details for the input medical abstract \{in a $\mathrm{selected\_tone}$\}. If there is no blanks, please output the original medical abstract.\\
& Please fill in the blanks in the following sentences to write an abstract of a medical research paper: \{$\mathrm{masked\_input}$\} and your answer MUST be exactly  $\{\mathrm{targeted\_word}\}$ words.
\\
\bottomrule
\end{tabular}
}
\end{table}

\subsubsection{API Prompt Designs}
In terms of \randomsampleapiname{},
\begin{itemize}[noitemsep,leftmargin=*]
    \item For \yelp{} data,  we generate 100 subcategories under each business category via ChatGPT and use them as keywords in the prompts.
    \item For \openreview{} data, we do not generate subcategories, as the review area label (e.g., ``\textit{Social Aspects of Machine Learning (eg, AI safety, fairness, privacy, interpretability, human-AI interaction, ethics)}'') already provides detailed information about the area. Instead, we generate a list of writers with their corresponing tones via ChatGPT (e.g., ``\textit{Postdoctoral Researcher: Advanced and knowledgeable insights'',
``AI Policy Maker: Concerned with regulatory and policy implications'',
``Robotics Engineer: Focus on practical applications in robotics''}) and use them as keywords in the prompt.
    \item For \pubmed{} data, we also generate a list of writers for medical abstracts via ChatGPT, such as ``\textit{Clinical Researcher, Principal Investigator, Biomedical Engineer}'', etc., and use them as keywords in the prompt.
\end{itemize} 
We provide the prompts of \randomsampleapiname{} for all datasets in \cref{tab:random_api} for \gptthreepointfive{} and \cref{tab:random_api_gpt2} for other LLMs.

In terms of \samplevariationapiname{}, (1) for GPT-3.5, we utilize fill-in-the-blanks with adaptive text lengths, providing few-shot demonstrations. To obtain \{$\mathrm{masked\_input}$\} used for fill-in-the-blanks, we calculate the tokens for  \{$\mathrm{input}$\} based on GPT-3.5 tokenizer\footnote{https://github.com/openai/tiktoken},  mask $p\%$ of them as blanks ``\_'', and decode them back to the text.
(2) In contrast, for \gpttwox{} models, we opt for zero-shot paraphrasing with fixed $\mathrm{\max\_token}$ as \samplevariationapiname{}. This choice is based on our observation that \gpttwox{} models do not follow the instructions of fill-in-the-blanks and adaptive text lengths well, as they are only pretrained on next-word-prediction tasks without further instruction tuning or reinforcement learning from human feedback (RLHF)~\cite{lambert2022illustrating} for blank filling tasks. Moreover, \gpttwox{} models do not gain much from few-shot demonstrations for paraphrasing, possibly due to their inferior instruction-following and in-context learning capabilities compared to GPT-3.5.

We provide the  prompts of \samplevariationapiname{} for \gpttwox{} models in \cref{tab:var_api_gpt2} and for GPT-3.5 in \cref{tab:var_api_gpt35_yelp},  \cref{tab:var_api_gpt35_openreview} and \cref{tab:var_api_gpt35_pubmed}.

\subsubsection{Differential Privacy.} 

Following \citet{yue2022synthetic}, we set $\delta=  \frac{1}{N_{\priv} \cdot \log(N_{\priv})}  $  for $(\epsilon,\delta)$-DP. 
As different datasets have different sizes of private training data, they require different $\delta$.  We run 10 PE iterations under DP on all datasets.
To achieve $\epsilon=\{1,2,4,\infty\}$, we use noise multiplier   $\sigma=\{15.34, 8.03, 4.24, 0\}$ for \yelp{};
$\sigma=\{11.60, 6.22, 3.38, 0\}$ for \openreview{};
$\sigma=\{13.26,7.01 ,3.75, 0\}$ for \pubmed{}.

\subsection{Implementation Details of Baselines.} 

For \ftgenerator{}, we finetune the \gpttwo{}-series models following the hyperparameters setup in Table 8 of  \cite{yue2022synthetic}.

For \ftdownstream{}, we report the hyperparameters for  \openreview{} and \yelp{} in  \cref{tb:hyper-downstream-yelp-openreview}, and \pubmed{} in \cref{tb:hyper-downstream-pubmed}.
For a target $\epsilon$, a noise multiplier is set as the smallest value such that DP-SGD can run the target number of steps.

\begin{table*}[ht]
\renewcommand{\arraystretch}{1.1}
\centering 
\small
\caption{Hyperparameters for \ftdownstream{} on \pubmed{}.} 
\vspace{-2mm}
\label{tb:hyper-downstream-pubmed}
\resizebox{0.9\linewidth}{!}{
\begin{tabular}{c|cc|ccccccccc}
\toprule
 &   \multicolumn{2}{c|}{\berttiny, \bertmini, \bertsmall{} for \pubmed{}}  &  \multicolumn{2}{c}{ \llamatwosevenb{} for \pubmed{}}   \\
 &  downstream (non-pri.)  & downstream (pri.)  &  downstream (non-pri.)  & downstream (pri.)  \\\midrule
Epoch & [5, 10, 30] & [10, 30, 50, 100] & 10 & 10 \\
Batch size & [32, 64] & [1024, 2048, 4096] & 128 & 128 \\
Clipping norm & - & [0.1, 0.5, 1, 3, 5] & - & 1 \\
Learning rate & [$3\times 10^{-5}$, \{$1$, $3$\}$\times 10^{-4}$] & [$3 \times 10^{-4}$, \{$1$, $3$\}$\times 10^{-3}$] & $1 \times 10^{-3}$ & $1 \times 10^{-3}$ \\
\bottomrule
\end{tabular}
}
\end{table*}

\begin{table*}[ht]
\renewcommand{\arraystretch}{1.1}
\centering 
\small
\caption{Hyperparameters for \ftdownstream{} on \yelp{} and \openreview{}.} 
\vspace{-2mm}
\label{tb:hyper-downstream-yelp-openreview}
\resizebox{0.9\linewidth}{!}{
\begin{tabular}{c|cc|ccccccccc}
\toprule
 &   \multicolumn{2}{c|}{\robertabase{} for \yelp{}}  &  \multicolumn{2}{c}{ \robertabase{} for \openreview{}}   \\
 &  downstream (non-pri.)  & downstream (pri.)  &  downstream (non-pri.)  & downstream (pri.)  \\\midrule
Epoch & [1,10] & [1,10] & 10 & 10 \\
Batch size & [128, 1024] & [128, 1024] & 8 & 128 \\
Clipping norm & - & 1 & - & 1 \\
Learning rate & $3\times 10^{-5}$ & $3\times 10^{-5}$ & $3\times 10^{-5}$ & $3\times 10^{-5}$ \\
\bottomrule
\end{tabular}
}
\end{table*}

\subsection{Metrics.} 
Here we provide more details about the metrics regarding embedding distribution distance.
We use sentence-transformer ``stsb-roberta-base-v2'' from HuggingFace\footnote{https://huggingface.co/models} to embed the real and synthetic datasets, and use seven evaluation metrics to measure embedding distribution distance: 
1) Fréchet Inception Distance (\textit{FID}) evaluates the feature-wise mean and covariance matrices of the embedding vectors and then computes the Fréchet distance between these two groups~\cite{heusel2017gans} ;
2) \textit{Precision} estimates the average sample quality;
3) \textit{Recall} assesses the breadth of the sample distribution;
4) \textit{F1} score is the harmonic mean of Precision and Recall, serving as a balance of the two~\cite{kynkaanniemi2019improved}; 
5) MAUVE evaluates the distributional distance of the synthetic and real data via divergence frontiers~\cite{pillutla2021mauve};
6) \textit{KL div.} measures the distance of embedding distributions based on KL divergence;
7) \textit{TV div.} quantifies the distance based on Total Variation divergence~\cite{chung1989measures}.

For downstream classification accuracy, we train downstream models \textbf{three times} and report the average accuracy.
For each metric associated with embedding distribution distance (except FID for which we use the whole dataset), we randomly draw 5000 samples (for efficiency) from the private dataset and the synthetic dataset respectively, to calculate the distance. 
We then report the averaged results based on \textbf{five} independent draws.

\section{Additional Experimental Results}
\label{app:exp-results}

\subsection{Robustness Against Membership Inference Attacks}
\label{app:mia_attack}
In this work, we focus on DP, a type of widely accepted privacy guarantee with profound theoretic backup which provides an upper bound for empirical membership privacy attacks.  To better understanding empirical risk, we perform state-of-the-art membership inference attacks (MIAs)~\cite{shokri2017membership} in the text domain.

We perform MIAs against the finetuned downstream models (which are fine-tuned on synthetic data for \name{}/ \ftgenerator{}; on real private data for \ftdownstream{}).  We randomly sample 4000 \pubmed{} real private samples as members and 4000 \pubmed{} test samples as non-members for evaluation.
We report AUC (Area Under the Curve) to evaluate the risks of MIAs. We consider three types of MIAs: 
   (1)  PPL thresholds perplexity to predict membership~\cite{carlini2021extracting}. 
    (2) REFER computes the ratio of the log perplexity of the tested model against a reference model~\cite{carlini2021extracting}. 
    (3) LIRA uses the ratio of likelihood instead of log-perplexity~\cite{carlini2022membership}. LiRA assumes the availability of high-quality data distributed similarly to the training set, which was thought to be impractical~\cite{tramer2022considerations}. Therefore, we follow \cite{mattern2023membership} to use the pre-trained model as a reference.

The results in \cref{tab:mia_auc} show that \name{} generally exhibits lower MIA AUC scores compared to both \ftgenerator{} and \ftdownstream{}  models. This indicates a higher robustness to empirical privacy attacks, potentially due to the synthetic nature of the data used for downstream model finetuning, which inherently reduces the risk of overfitting to real private data.

\subsection{Convergence of Text Length Distribution}
\label{app:sentence-len-converge}
As shown in \cref{fig:gpt35-length-conv-yelp}, \cref{fig:gpt35-length-conv-pubmed} and \cref{fig:gpt35-length-conv-openreview}\footnote{For \openreview{} in \cref{fig:gpt35-length-conv-openreview}, we use a temperature of 1.4.},  we see that over the PE iterations, the text length distribution of synthetic samples produced from \chatgpt{} through our \name{} converges, as it becomes closer to the distribution of the original data. This showcases the effectiveness of our adaptive text length mechanism.  
We note that there is a noticeable peak near 30 tokens for our synthetic texts on \yelp{}, which is attributed to the $\mathrm{\min\_word}$ used in the \samplevariationapiname{}  prompt to avoid generating blank outputs.  

\begin{figure}
    \centering
    \includegraphics[width=1.0\linewidth]{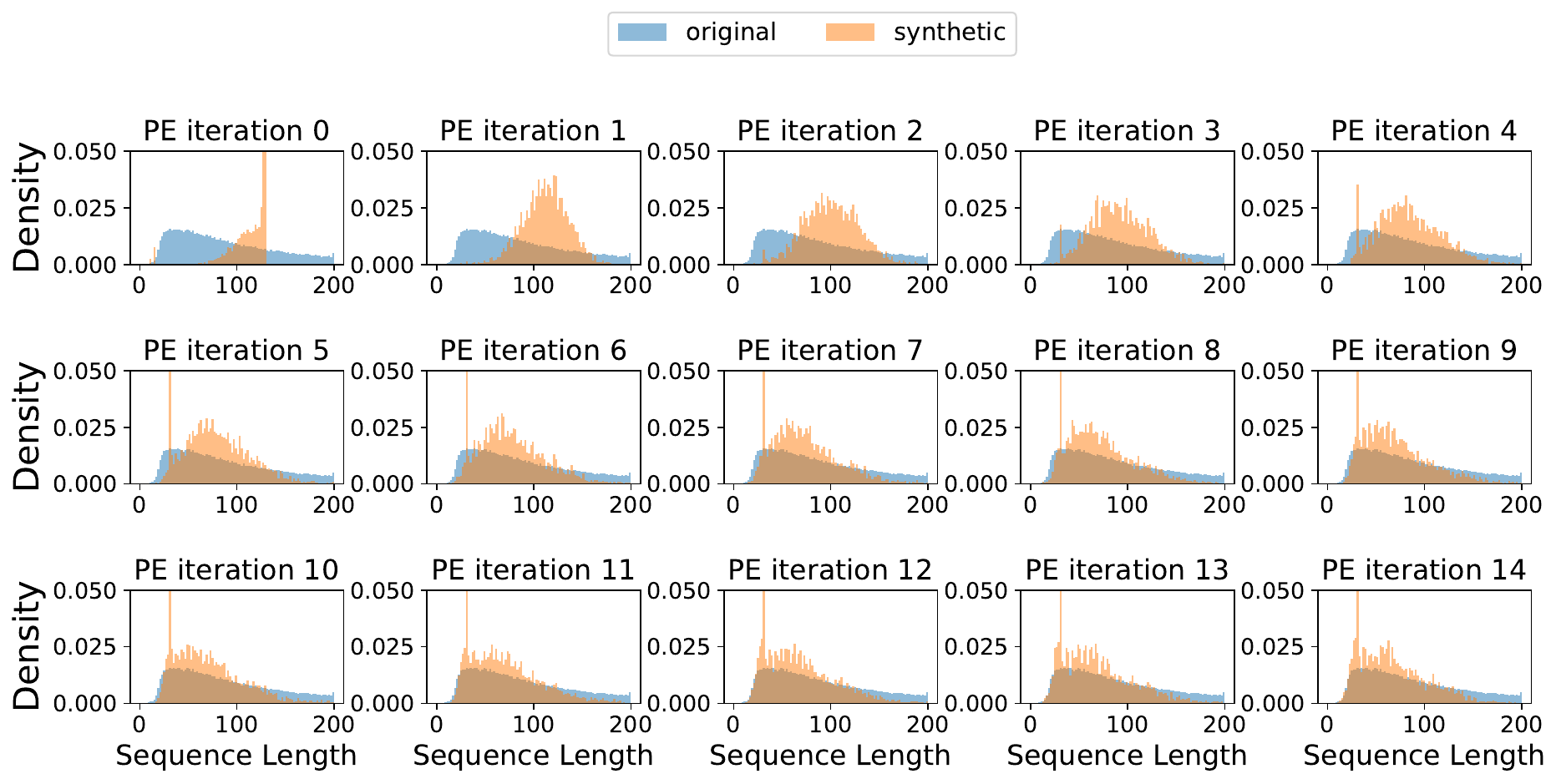} 
    \caption{\small Convergence of text length distribution over \name{} iterations on \yelp{} synthetic text generated from \gptthreepointfive{}.
    }
    \label{fig:gpt35-length-conv-yelp}
\end{figure}

\begin{figure}
    \centering
    \includegraphics[width=1.0\linewidth]{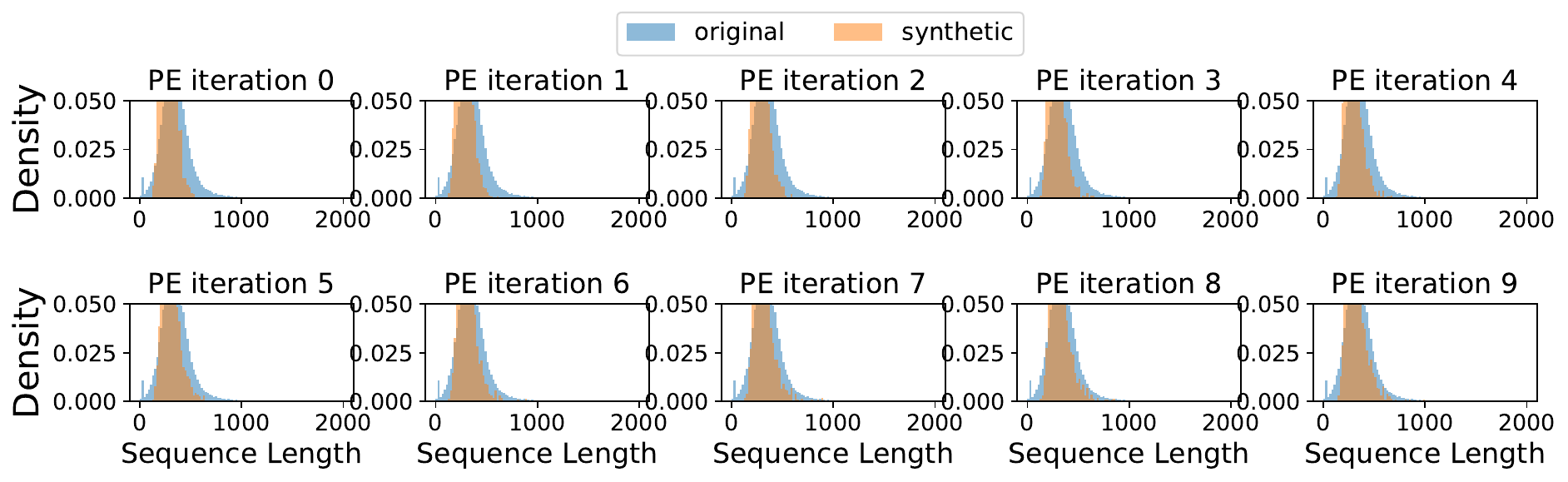} 
    \caption{\small Convergence of text length distribution over \name{} iterations on \pubmed{} synthetic text generated from \gptthreepointfive{}.
    }
    \label{fig:gpt35-length-conv-pubmed}
\end{figure}

\begin{figure}
    \centering
    \includegraphics[width=1.0\linewidth]{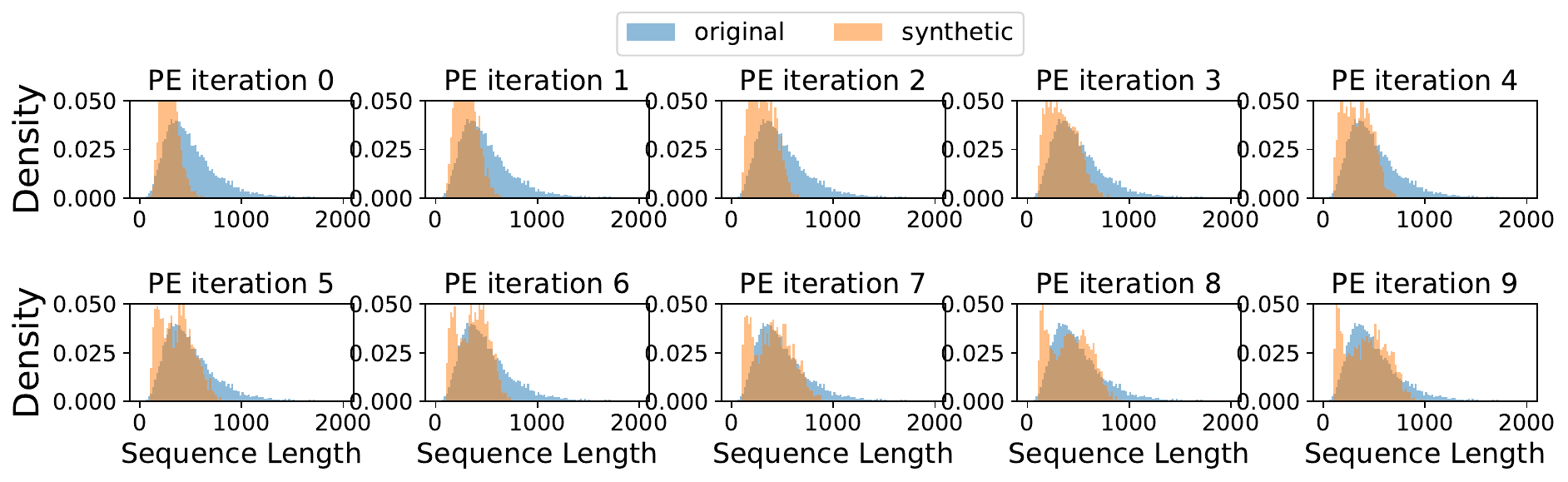} 
    \caption{\small Convergence of text length distribution over \name{} iterations on \openreview{} synthetic text generated from \gptthreepointfive{}.
    }
    \label{fig:gpt35-length-conv-openreview}
\end{figure}

\subsection{Efficiency in Terms of GPU Hours}
\begin{table}[h]
\caption{GPU hours on one 32G NVIDIA V100 for \name{} \ftgenerator{}  on \yelp{} under $\epsilon=1$. \name{} is more efficient with fewer total GPU hours. }
\label{tab:gpu_hours}
\centering
\resizebox{0.8\columnwidth}{!}{%
\begin{tabular}{ll|c|cccc}\toprule
&     & \multirow{2}{*}{DP-SGD finetune}  &  \multicolumn{3}{c}{Generation} \\
&     &    &  5k samples &  10k samples &  100k samples \\\midrule
\multirow{3}{*}{\begin{tabular}[c]{@{}l@{}} \ftgenerator{}  \end{tabular}} & GPT2& 456.71 & 0.22 & 0.45  & 4.47    \\
& GPT2-Medium & 709.50 & 0.25 & 0.50  & 5.03   \\
& GPT2-large  & 1764.42& 0.35 & 0.70  & 6.96   \\\midrule
\multirow{3}{*}{\name{}  ($L=2$)}     & GPT2& /      & 1.76 & 2.48 & 13.35 \\
& GPT2-Medium & /      & 2.30 & 2.89 & 18.68   \\
& GPT2-large  & /      & 2.68 & 3.83 & 26.98 \\\midrule
\multirow{3}{*}{\name{} ($L=7$) }     & GPT2& /      & 6.04 & 9.07  & 66.66  \\
& GPT2-Medium & /      & 6.94 & 11.55 & 91.07  \\
& GPT2-large  & /      & 9.62 & 16.77 & 139.35\\
 \bottomrule
\end{tabular}
}
\end{table}

In \cref{tab:gpu_hours}, we provide a detailed breakdown of the GPU hours shown in \cref{fig:gpu-hour-yelp}. 
We consider the process of generating DP synthetic data given a private dataset. \ftgenerator{} \cite{yue2022synthetic} requires two steps: (1) finetuning a pretrained data generator with DP-SGD, and (2) generating samples from the finetuned data generator, whereas \name{} requires only one step (\cref{algo}). In \cref{tab:gpu_hours}, we list the GPU hours of each step of each method. For \citet{yue2022synthetic}, we use the hyper-parameters in their Table 8.

We can see that the majority of the time spent by \ftgenerator{} is the DP-fine-tuning stage, which is already much more costly than the total cost of \name{}. This results from two factors: (1) Training is costly due to the backpropagation, especially for large models; (2) DP-SGD requires per-sample gradients, which further increases the memory and computation cost. In contrast, \name{} only requires model inference and does not require model training, and is thus more efficient.

It is also worth noting that once the model is DP finetuned, \ftgenerator{} can efficiently generate many samples with only model inference. It is illustrated by the small GPU hours in the ``Generation'' step of \ftgenerator{}. In contrast, in \name{}, the required GPU hour is positively correlated with the number of samples. Therefore, \ftgenerator{} can become more efficient than \name{} when the number of generated samples is large enough. However, the original \pename{} paper \cite{lin2023differentially} proposed an efficient way to generate more DP samples after \pename{} is done, by passing the generated samples through \samplevariationapiname{}. In the context of text generation with LLMs, this approach is expected to have a similar overhead as generating more samples from the DP-finetuned generator in \ftgenerator{}. We defer the study of this approach to future work.

\subsection{Comparison Between \name{} and \pename{}}
\label{app:compare_pe_augpe}

We compare \name{} against \pename{} when using \gptthreepointfive{} as the generator on three datasets.
The results in  \cref{tab:pe_augpe_gpt3.5} show that \name{} is always better than \pename{} on \pubmed{} for \gptthreepointfive{}. Moreover, \name{} is better for  \openreview{} \openreviewrecom{} classification task and \yelp{} \yelprating{} classification task.
As \name{} supports \pename{} as a special case by changing the hyperparameters of $L$ and $K$, the practitioner can adjust those hyperparameters for a specific downstream task and find the best settings to generate synthetic data.

\begin{table}[h]
\caption{Comparision between \name{} and \pename{} when using \gptthreepointfive{} as generator on three datasets.}
\label{tab:pe_augpe_gpt3.5}
\resizebox{\columnwidth}{!}{%
\begin{tabular}{ll|cc|cc|cc|ccc}\toprule
\multirow{1}{*}{Data Type (Size)} & \multirow{1}{*}{Method}  & \multicolumn{2}{c}{$\epsilon=\infty$} & \multicolumn{2}{|c|}{$\epsilon=4$} & \multicolumn{2}{c|}{$\epsilon=2$} & \multicolumn{2}{c}{$\epsilon=1$}    \\ \midrule
\yelp{}  &  &    \yelprating & \yelpcategory &  \yelprating & \yelpcategory   &  \yelprating & \yelpcategory   &  \yelprating & \yelpcategory \\\midrule
Synthetic (5000) & \pename{} $\gets$ \name{} ($k=3,L=1$)  & 67.9 & \textbf{74.7}  & 67.1 & \textbf{74.6} &  67.2 & \textbf{74.6}  &  67.6 & \textbf{74.7}  \\
\rowcolor{lightgray}
Synthetic (5000) &  \name{} ($k=0,L=4$) &  \textbf{68.4} &  74.1  & \textbf{68.1} & 74.0 & \textbf{67.8} & 74.3 & \textbf{67.9} & 74.0  \\ \midrule\midrule
\openreview{} & &   \openreviewarea & \openreviewrecom &  \openreviewarea & \openreviewrecom &  \openreviewarea & \openreviewrecom &  \openreviewarea & \openreviewrecom \\\midrule
 Synthetic (2000) & \pename{} $\gets$ \name{} ($k=3,L=1$)    &  43.6 & 42.4 & \textbf{43.6} & 43.5  & \textbf{44.6} &  43.7  &  \textbf{42.0} &  42.9 \\
 \rowcolor{lightgray} 
 Synthetic (2000) &  \name{} ($k=0,L=4$)    &  \textbf{45.4}  & \textbf{43.5} & 43.5 & \textbf{44.6} & 42.8  & \textbf{44.5} & 41.9 & \textbf{43.1}  \\\midrule\midrule
\pubmed{} & &  \bertmini & \bertsmall & \bertmini & \bertsmall & \bertmini & \bertsmall & \bertmini & \bertsmall\\\midrule
Synthetic (2000) & \pename{} $\gets$ \name{} ($k=3,L=1$) & 29.7& 31.8 & 29.6  & 31.8 & 29.7 & 31.9 & 29.8 & 31.9  \\
 \rowcolor{lightgray}
 Synthetic (2000) & \name{} ($k=0,L=4$)  & \textbf{30.4} & \textbf{32.7} & \textbf{30.3} & \textbf{32.5} & \textbf{30.2} & \textbf{32.5}   &\textbf{ 30.1} & \textbf{32.4} \\
 \bottomrule
\end{tabular}
}
\end{table}

\subsection{Ablation Study on Variation API Prompt Design}

\begin{table}[h]
\caption{Evaluation on Variation API designs for \gpttwo{} and \gptthreepointfive{} on \yelp{}. Fill-in-the-blanks is prefered for \gptthreepointfive.}
\centering
\label{tab:effect_variation_api}
\resizebox{0.6\columnwidth}{!}{%
\begin{tabular}{l|ll|ccccccccc}\toprule
\multirow{2}{*}{Variation API prompt} &  \multicolumn{2}{c}{GPT-2} &  \multicolumn{2}{|c}{GPT-3.5}    \\ 
&    Rating & Category &    Rating & Category  \\\midrule
paraphrasing & 67.5 & 74.8 &  67.5 & 74.3\\
paraphrasing  w/ few-shot demos & \textbf{67.8} & 73.6 & 65.7 & 74.2 \\
fill-in-the-blanks & 66.3  & \textbf{74.6} & \textbf{67.9} &  74.6 \\
 fill-in-the-blanks w/ few-shot demos & 67.6  & 74.8    & \textbf{67.9} &  \textbf{74.7}\\
 \bottomrule
\end{tabular}
}
\end{table}

The results in \cref{tab:effect_variation_api} show that fill-in-the-blanks prompt (with few-shot demonstrations) yields better results for \gptthreepointfive{}.
For \gpttwo{},  paraphrasing can be an effective strategy. Although fill-in-blanks leads to high accuracy on \yelp{} \yelpcategory{} classification task, we find that the generated texts have many unfilled blanks ``\_\_'' upon inspection.

\subsection{Leveraging Open-source LLMs as Generator for \name{}}
We use opensource LLMs from Huggingface as data generators in \name{}. We find that \llamatwosevenb{} does not follow the fill-in-the-blank prompts well and often
leaves blanks (``\_\_'') in the generated texts. Also, it struggles to adhere to the word prompt and the length of synthetic sequences exhibits a large gap from the targeted word  specified in the prompt.  It might be because they are not explicitly instruction/RLHF-tuned for those blank-filling and word count tasks, and have inferior instruction-following and in-context learning capabilities compared to GPT-3.5.
Therefore, we turn to use the same hyperparameter setup as \gpttwox{} models for those open-source LLMs. 
The results in \cref{tab:effect_data_generator} show that \gptthreepointfive{} outperforms most of the  models on \pubmed{} tasks and  \openreview{} \openreviewrecom{} classification task by a large margin. For \openreview{} \openreviewarea{} task, Mixtral-8x7B-v0.1 is better than \gptthreepointfive{}, demonstrating the competitive generation power of Mixtral-8x7B-v0.1 for academic reviews in machine learning domains.

\subsection{Effect of Rank-based Sampling}
\label{app:exp_rank_ablation}

\begin{table}[h]
\caption{Comparing rank-based sampling against probability-based random sampling for \name{} with \gpttwo{}-series models on three datasets.}
\label{tab:rank_full}
\centering
\resizebox{0.6\columnwidth}{!}{%
\begin{tabular}{ll|cc|cc}
\toprule
\multirow{1}{*}{Data Type (Synthetic Data Size)} & \multirow{1}{*}{Data Generator}   & Rank & Prob & Rank & Prob \\ \midrule
&    & \multicolumn{2}{c|}{ \yelprating }   & \multicolumn{2}{c}{\yelpcategory}  \\
\cmidrule(lr){3-4} \cmidrule(lr){5-6}
\multirow{3}{*}{Yelp (5000)} & GPT-2 & \textbf{67.5} & 66.7 & \textbf{74.8} & 74.7 \\
 & GPT-2 Medium & 67.5 & \textbf{67.7} & \textbf{74.9 }& 74.6 \\
 & GPT-2 Large & \textbf{67.5} & 67.1 & \textbf{74.5 }& 74.4 \\
\midrule
 &   &  \multicolumn{2}{c|}{\openreviewarea} & \multicolumn{2}{c}{\openreviewrecom} \\
\cmidrule(lr){3-4} \cmidrule(lr){5-6}
\multirow{3}{*}{OpenReview (2000)} & GPT-2 & \textbf{42.4 }& 39.8 & 32.1 & 32.1 \\
 & GPT-2 Medium & \textbf{41.0} & 37.1 & \textbf{32.3} & 32.0 \\
 & GPT-2 Large & \textbf{42.1} & 40.1 & \textbf{32.1 }& 32.0 \\
\midrule
 &   &  \multicolumn{2}{c|}{\bertmini} & \multicolumn{2}{c}{\bertsmall} \\
\cmidrule(lr){3-4} \cmidrule(lr){5-6}
\multirow{3}{*}{PubMed (2000)} & GPT-2 & \textbf{24.5} & 23.4 & \textbf{26.7} & 25.4 \\
 & GPT-2 Medium & \textbf{25.5} & 23.9 & \textbf{27.7} & 25.9 \\
 & GPT-2 Large & \textbf{25.7} & 24.1 & \textbf{28.0} & 26.0 \\
\bottomrule
\end{tabular}
}
\end{table}
We compare our proposed rank-based sampling (\cref{line:directuse}) against probability-based random sampling in the original \pename{} (\cref{line:drawfromhist}) across \gpttwo{}, \gpttwom{} and \gpttwol{} on three datasets. 
The results in \cref{tab:rank_full} indicate that our proposed rank-based sampling (\cref{line:directuse}) consistently outperforms probability-based random sampling in the original \pename{} (\cref{line:drawfromhist}), due to the elimination of sample redundancy inherent in random sampling, as rank-based sampling exclusively selects the top $N_{\syn}$ samples.

\subsection{Effect of Iteration $T$ on DP Utility}
\label{app:effect_t_dp}
\cref{tab:effect_private_data} presents the results on the non-DP setting, serving as an ablation study to underscore the role of private data in \name{}. In the DP setting, given a fixed privacy budget, a larger $T$ requires more noise, which may compromise the utility of the DP histogram. On the other hand, a larger $T$ 
allows for more iterations of sample improvement. %
We want to study the joint effect of these two factors.

We conducted experiments comparing the utility of the algorithm at  $t=1$ and $t=10$ on three datasets under $\epsilon=4,2,1$.  The results in \cref{tab:effect_t_epsilon4,tab:effect_t_epsilon2,tab:effect_t_epsilon1} show that $t=10$ consistently yields better utility than $t=1$ across all three privacy budget levels, underscoring the effectiveness of \name{}'s iterative improvement mechanism. This finding suggests that, despite the increased noise, the algorithm can robustly preserve useful statistical properties in the DP histogram and generate high-quality texts under $t=10$.

\begin{table}[ht]
\centering
\caption{Effect of \name{} iteration $t$ on the DP utility under $\epsilon=4$.}
\begin{tabular}{l|cc|cc|cc}
\toprule
 & \multicolumn{2}{c|}{\yelp{}} & \multicolumn{2}{c|}{\openreview{}} & \multicolumn{2}{c}{\pubmed{}} \\
\midrule
 & \yelprating & \yelpcategory & \openreviewarea & \openreviewrecom & \bertmini & \bertsmall \\
\midrule
$t=1$ & 64.7 & 73.5 & 36.5 & 42.1 & 29.9 & 32.3 \\
$t=10$ & 67.8 & 74.6 & 43.5 & 44.6 & 30.3 & 32.5 \\
\bottomrule
\end{tabular}
\label{tab:effect_t_epsilon4}
\end{table}

\begin{table}[ht]
\centering
\caption{Effect of \name{} iteration $t$ on the DP utility under $\epsilon=2$.}
\begin{tabular}{l|cc|cc|cc}
\toprule
 & \multicolumn{2}{c|}{\yelp{}} & \multicolumn{2}{c|}{\openreview{}} & \multicolumn{2}{c}{\pubmed{}} \\
\midrule
 & \yelprating & \yelpcategory & \openreviewarea & \openreviewrecom & \bertmini & \bertsmall \\
\midrule
$t=1$ & 63.9 & 73.6 & 37.2 & 42.0 & 29.9 & 32.3 \\
$t=10$ & 67.4 & 74.3 & 42.8 & 44.5 & 30.2 & 32.5 \\
\bottomrule
\end{tabular}
\label{tab:effect_t_epsilon2}
\end{table}

\begin{table}[ht]
\centering
\caption{Effect of \name{} iteration $t$ on the DP utility under $\epsilon=1$.}
\begin{tabular}{l|cc|cc|cc}
\toprule
 & \multicolumn{2}{c|}{\yelp{}} & \multicolumn{2}{c|}{\openreview{}} & \multicolumn{2}{c}{\pubmed{}} \\
\midrule
 & \yelprating & \yelpcategory & \openreviewarea & \openreviewrecom & \bertmini & \bertsmall \\
\midrule
$t=1$ & 63.8 & 73.1 & 37.4 & 41.7 & 30.0 & 32.3 \\
$t=10$ & 66.8 & 74.7 & 41.9 & 43.1 & 30.1 & 32.4 \\
\bottomrule
\end{tabular}
\label{tab:effect_t_epsilon1}
\end{table}

\subsection{Embedding Distribution Distance Between Real and Synthetic data}
\label{app:emb_distance}

We report the results of embedding distribution distance between real and synthetic data on \yelp{} in \cref{fig:emd-dist-gpt2-compare-yelp-10k}, and on \pubmed{} in \cref{fig:emd-dist-gpt2-compare-pubmed-2k}. 
When using the same base model \gpttwo{} for a fair comparison,  we observe that under DP and non-DP settings, 
\name{} can obtain similar and even lower embedding distribution distances between real and synthetic samples for certain metrics compared to fine-tuning.
For example, on \yelp{} dataset, under DP, \name{} yields better FID, precision, recall, F1 than \ftgenerator{} and achieves comparable  MAUVE scores. 
On \pubmed{} dataset, under DP, \name{} yields better FID, MAUVE scores, KL divergence, and TV divergence than \ftgenerator{}. 
These findings highlight the promise %
of employing the API-only method for DP synthetic text generation.

\begin{figure*}[t]
    \centering
    \includegraphics[width=1\linewidth]{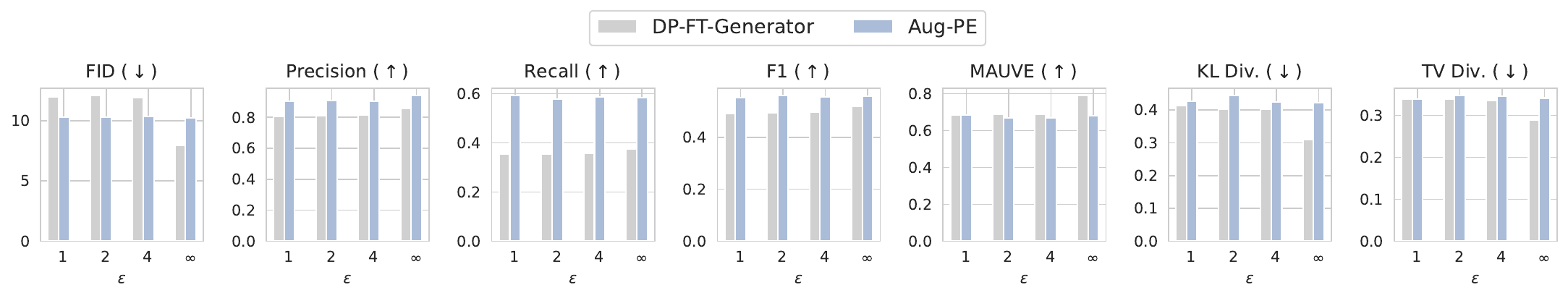} 
     \caption{Evaluation on distribution distances between \yelp{} real data and \gpttwo{} generated 10k DP synthetic samples.}
    \label{fig:emd-dist-gpt2-compare-yelp-10k}
    \vspace{-2mm}
\end{figure*}

\begin{figure*}[t]
    \centering
    \includegraphics[width=1\linewidth]{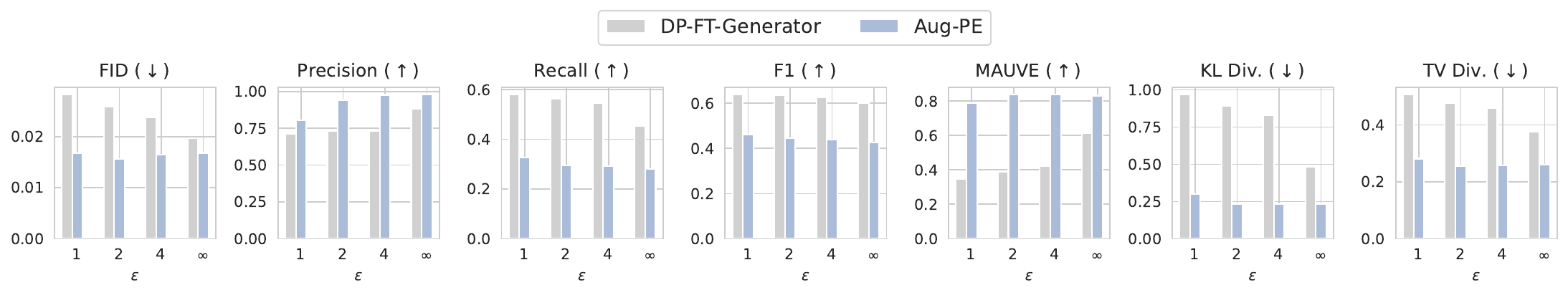} 
    \caption{Evaluation on distribution distances between \pubmed{} real data and \gpttwo{} generated 2k DP synthetic samples.
    }
    \label{fig:emd-dist-gpt2-compare-pubmed-2k}
    \vspace{-2mm}
\end{figure*}

\subsection{Downstream Task Utility Under Various Synthetic Data Size}
\label{app:downstream_task_utility_full}

\subsubsection{Utility on \yelp{}}

We report the full results of downstream accuracy  on \yelp{} in \cref{tb:yelp_utility_full}.
We find that
(1) when using the same base model %
for a fair comparison, we see that under DP settings, \name{} demonstrates competitive (or even better) utility on downstream classification tasks compared to fine-tuning.  The scores are also close to that of the downstream algorithms trained on the real data under DP directly, demonstrating the promise of DP synthetic text as a tool for DP machine learning.
(2) For large models like \gpttwol{} and \gpttwom{}, more synthetic samples (e.g., 100k) from \name{} can enhance downstream utility.
However, for \gpttwo{}, sometimes 10k synthetic samples can lead to better downstream utility than 100k samples, which might be due to the low-quality data generated from the small model that hurts the performance.

\begin{table}[h]
\caption{Classification accuracy of downstream \robertabase{} model under $\epsilon=\infty, 4, 2, 1$ on \yelp{} for two downstream tasks: review rating and business category classification.
(i) Compared to \ftgenerator{}, 
    in some cases, downstream accuracy of \name{} is higher  (\ua{})  under the same synthetic data size and the same \gpttwo-series data generator. Leveraging the inherent knowledge within stronger LLM, \gptthreepointfive, \name{} can achieve higher accuracy. (ii) Compared to traditional method \ftdownstream{}, \name{} can also obtain higher accuracy under DP with the same synthetic data size.
}
\label{tb:yelp_utility_full}
\resizebox{\columnwidth}{!}{%
\begin{tabular}{llllllllllll}\toprule
\multirow{2}{*}{Data Type (Size)} & \multirow{2}{*}{Method} & \multirow{2}{*}{Data Generator} & \multicolumn{2}{c}{$\epsilon=\infty$} & \multicolumn{2}{c}{$\epsilon=4$} & \multicolumn{2}{c}{$\epsilon=2$} & \multicolumn{2}{c}{$\epsilon=1$}    \\ 
 &  &  & Rating & Category & Rating & Category   & Rating & Category   & Rating & Category \\\midrule
  Original (1,939,290) & \ftdownstream  & - & 76.0 & 81.6 & 67.5 & 72.8 & 67.2 & 72.0 & 66.8 & 71.8 \\
Original (100,000)   & \ftdownstream   &  - &   72.7 & 75.5  &    65.0  & 71.2 & 64.1 & 70.0 & 62.9 & 68.7  \\
Original (10,000)   & \ftdownstream  & -    & 70.9 & 76.2 & 44.8 & 61.8 & 44.8 & 61.8  & 44.8 & 61.8    \\
Original (5,000)   & \ftdownstream   &  - &  70.5 & 75.1 & 44.8 & 61.8 & 44.8 & 61.8 & 44.8 & 61.8 \\
\midrule

Synthetic (5000) & \ftgenerator{} & \gpttwo & 70.3 & 75.9 & 68.2 & 74.1 & 67.2 & 73.1 & 66.4 & 73.9 \\

Synthetic (10000) & \ftgenerator{} & \gpttwo & 71.1 & 75.8 & 68.2 & 73.0 & 67.7 & 73.2 & 66.7 & 73.7 \\

Synthetic (100000) & \ftgenerator{} & \gpttwo & 71.0 & 75.6 & 66.8 & 72.6 & 67.0 & 72.3 & 65.5 & 71.8 \\

\rowcolor{lightgray} Synthetic (5000) & \name{} & \gpttwo & 67.5 & 74.8 & 66.4 & 74.9 \ua{} & 67.1 & 74.7 \ua{} & 66.9 \ua{} & 74.4 \ua{} \\
\rowcolor{lightgray} Synthetic (10000) & \name{} & \gpttwo & 67.2 & 75.1 & 66.6 & 75.3 \ua{} & 66.2 & 74.9 \ua{} & 66.0 & 74.6 \ua{} \\
\rowcolor{lightgray} Synthetic (100000) & \name{} & \gpttwo & 67.1 & 76.0 \ua{} & 66.3 & 75.1 \ua{} & 66.1 & 75.0 \ua{} & 65.7 \ua{} & 74.5 \ua{} \\
\midrule

Synthetic (5000) & \ftgenerator{} & \gpttwom & 70.0 & 75.0 & 69.1 & 74.6 & 67.8 & 74.3 & 67.4 & 74.1 \\

Synthetic (10000) & \ftgenerator{} & \gpttwom & 70.7 & 75.6 & 68.8 & 74.4 & 68.2 & 73.8 & 67.5 & 73.9 \\

Synthetic (100000) & \ftgenerator{} & \gpttwom & 71.9 & 76.3 & 68.1 & 73.9 & 67.8 & 74.3 & 67.9 & 73.3 \\
\rowcolor{lightgray} Synthetic (5000) & \name{} & \gpttwom & 67.5 & 74.9 & 66.8 & 74.6 & 67.8 & 74.7 \ua{} & 67.4 & 74.6 \ua{} \\
\rowcolor{lightgray} Synthetic (10000) & \name{} & \gpttwom & 67.5 & 74.9 & 67.4 & 74.9 \ua{} & 67.6 & 75.1 \ua{} & 67.1 & 74.7 \ua{} \\
\rowcolor{lightgray} Synthetic (100000) & \name{} & \gpttwom & 68.2 & 75.8 & 67.4 & 75.5 \ua{} & 66.6 & 75.3 \ua{} & 66.2 & 74.7 \ua{} \\
\midrule

Synthetic (5000) & \ftgenerator{} & \gpttwol & 70.4 & 75.4 & 68.7 & 74.2 & 69.8 & 75.1 & 68.7 & 74.6 \\

Synthetic (10000) & \ftgenerator{} & \gpttwol & 70.7 & 74.3 & 69.2 & 74.9 & 69.7 & 75.2 & 68.9 & 74.6 \\

Synthetic (100000) & \ftgenerator{} & \gpttwol & 71.8 & 74.1 & 69.5 & 74.5 & 68.7 & 74.5 & 69.6 & 74.4 \\
\rowcolor{lightgray} Synthetic (5000) & \name{} & \gpttwol & 67.5 & 74.5 & 67.3 & 74.4 \ua{} & 65.8 & 74.1 & 66.6 & 75.0 \ua{} \\
\rowcolor{lightgray} Synthetic (10000) & \name{} & \gpttwol & 67.1 & 74.7 \ua{} & 67.1 & 74.9 & 66.6 & 74.7 & 67.0 & 74.4 \\
\rowcolor{lightgray} Synthetic (100000) & \name{} & \gpttwol & 67.3 & 75.8 \ua{} & 67.6 & 75.7 \ua{} & 66.8 & 75.4 \ua{} & 66.0 & 75.3 \ua{} \\
\midrule
\rowcolor{lightgray} Synthetic (5000) & \name{} & GPT-3.5 & 68.4 & 74.1 & 68.1 & 74.0 & 67.8 & 74.3 & 67.9 & 74.0\\ 
 \bottomrule
\end{tabular}
}
\end{table}

\subsubsection{Utility on \openreview{}}
We report the downstream accuracy on \openreview{} in \cref{tb:openreview-utility}.
The key observations are:
(1) Under DP when using the same \gpttwo{}/\gpttwom{}/\gpttwol{} as the base model, \name{} achieve similar classification accuracy and %
classification accuracy compared with \ftgenerator{}. This again demonstrates that \name{} is a promising alternative to DP fine-tuning.
(2) More synthetic samples lead to better area classification accuracy for the three \gpttwox{} models, indicating that \name{} scales well with the synthetic sample size. Note that  both \name{} and \ftgenerator{} do not perform well on review rating classification tasks across different data sizes, which shows the inherent limitation of \gpttwox{} models -- they may struggle to generate academic texts with correct sentiments. 
(3) \name{} with \chatgpt{} achieves %
better utility than \name{} with \gpttwol{} on both tasks  with or without DP. This suggests that \name{} benefits from larger and more powerful LLMs. We expect that as the capability of LLMs quickly evolves, \name{} can be even more promising in the future.
(4) However, there is still a gap between the results of \name{} under non-DP setting $\epsilon=\infty$ and the results on the original data. This suggests that even in the non-DP setting, \name{} is still not able to recover the distribution of the real data. This gap is unavoidable in the DP setting. We hypothesize that better hyper-parameter tunings (e.g., the variation degree) could lower the gap. We leave a more careful investigation of this issue to future work.

\begin{table}[h]
\caption{Classification accuracy of downstream \robertabase{} model under $\epsilon=\infty, 4, 2, 1$ on \openreview{} for two downstream tasks: review area and rating classification.
(i) Compared to \ftgenerator{}, 
    in some cases, downstream accuracy of \name{} is higher  (\ua{})  under the same synthetic data size and the same \gpttwo-series data generator. Leveraging the inherent knowledge within stronger LLM, \gptthreepointfive, \name{} can achieve higher accuracy. (ii) Compared to traditional method \ftdownstream{}, \name{} can also obtain higher accuracy under DP with the same synthetic data size.
    }
\label{tb:openreview-utility}
\resizebox{\columnwidth}{!}{%
\begin{tabular}{lllllllllllll}\toprule
\multirow{2}{*}{Data Type (Size)} & \multirow{2}{*}{Method} & \multirow{2}{*}{Data Generator} & \multicolumn{2}{c}{$\epsilon=\infty$} & \multicolumn{2}{c}{$\epsilon=4$} & \multicolumn{2}{c}{$\epsilon=2$} & \multicolumn{2}{c}{$\epsilon=1$}    \\
 &  &  & Area & Rating & Area & Rating   & Area & Rating   & Area & Rating \\\midrule
\multirow{1}{*}{Original (8396)} & \ftdownstream & - & 65.2 & 50.9 & 30.5 & 32.0 & 30.5 & 32.0 & 30.5 & 32.0 \\
\multirow{1}{*}{Original (2000)} & \ftdownstream  & - & 55.3 &  47.8 & 30.5 & 32.0 & 30.4 & 25.5 & 6.3 & 19.8 \\\midrule

Synthetic (2000) & \ftgenerator{} & \gpttwo & 47.5 & 32.0 & 32.1 & 32.0 & 31.9 & 32.0 & 32.1 & 32.0  \\

Synthetic (3000) & \ftgenerator{} & \gpttwo & 48.0 & 32.0 & 34.1 & 32.0 & 33.6 & 32.0 & 33.6 & 32.0  \\
   
Synthetic (5000) & \ftgenerator{} & \gpttwo & 48.3 & 35.8 & 32.7 & 32.0 & 30.5 & 32.0 & 35.6 & 31.1 & \\
\rowcolor{lightgray} Synthetic (2000) & \name{} & \gpttwo & 42.4 & 32.1 \ua{} & 39.9 \ua{} & 32.1 \ua{} & 38.8 \ua{} & 32.1 \ua{} & 37.6 \ua{} & 32.0  \\
\rowcolor{lightgray} Synthetic (3000) & \name{} & \gpttwo & 43.2 & 32.0 & 39.1 \ua{} & 32.0 & 38.6 \ua{} & 32.1 \ua{} & 39.5 \ua{} & 32.1 \ua{}  \\
\rowcolor{lightgray} Synthetic (5000) & \name{} & \gpttwo & 43.4 & 32.1 & 40.1 \ua{} & 32.0 & 39.2 \ua{} & 32.0 & 37.9 \ua{} & 32.0 \ua{}  \\
\midrule 

Synthetic (2000) & \ftgenerator{} & \gpttwom & 49.7 & 36.5 & 40.3 & 32.0 & 33.5 & 31.9 & 35.6 & 31.9  \\

Synthetic (3000) & \ftgenerator{} & \gpttwom & 50.6 & 38.7 & 38.4 & 32.0 & 36.5 & 31.3 & 33.1 & 30.6  \\

Synthetic (5000) & \ftgenerator{} & \gpttwom & 50.3 & 41.2 & 39.8 & 31.4 & 37.4 & 31.7 & 34.6 & 31.0  \\
\rowcolor{lightgray} Synthetic (2000) & \name{} & \gpttwom & 41.0 & 32.3 & 36.9 & 32.0 & 36.0 \ua{} & 32.0 \ua{}  & 36.6 \ua{} & 32.1  \ua{} \\
\rowcolor{lightgray} Synthetic (3000) & \name{} & \gpttwom & 42.1  & 32.1 & 38.3 & 32.1 \ua{} & 38.9 \ua{} & 32.1 \ua{} & 37.5 \ua{} & 32.1 \ua{}  \\
\rowcolor{lightgray} Synthetic (5000) & \name{} & \gpttwom & 43.5 & 32.5 & 37.5 & 32.0 \ua{} & 35.5 & 32.0 \ua{} & 36.8 \ua{} & 32.1 \ua{}  \\
\midrule
Synthetic (2000) & \ftgenerator{} & \gpttwol & 48.3 & 42.9 & 38.9 & 33.7 & 40.4 & 33.6 & 38.6 & 32.2 & \\

Synthetic (3000) & \ftgenerator{} & \gpttwol & 49.8 & 43.7 & 41.3 & 33.9 & 42.8 & 31.6 & 38.2 & 32.7 & \\

Synthetic (5000) & \ftgenerator{} & \gpttwol & 52.5 & 44.5 & 42.0 & 34.2 & 41.7 & 34.9 & 40.1 & 32.8  \\

\rowcolor{lightgray} Synthetic (2000) & \name{} & \gpttwol & 42.1 & 32.1 & 38.8 & 32.0 & 38.4 & 32.0 & 38.1 & 32.0  \\

\rowcolor{lightgray} Synthetic (3000) & \name{} & \gpttwol & 44.0 & 32.1 & 39.7 & 32.2 & 38.4 & 32.1 \ua{} & 36.4 & 32.0  \\

\rowcolor{lightgray} Synthetic (5000) & \name{} & \gpttwol & 44.1 & 32.1 & 39.3 & 32.1 & 39.5 & 32.1 & 37.4 & 32.1  \\
\midrule
\rowcolor{lightgray} Synthetic (2000) & \name{} & \gptthreepointfive{} & 45.4 & 43.5 & 43.5 & 44.6 & 42.8 & 44.5 & 41.9 & 43.1  \\
\bottomrule
\end{tabular}
}
\end{table}

\begin{table}[h]
\caption{Next word prediction accuracy of downstream \bertmini{} model under $\epsilon=\infty, 4, 2, 1$ on \pubmed{}.
(i) Compared to \ftgenerator{},   \name{} with a strong LLM  \gptthreepointfive{} can achieve higher accuracy under DP with the same synthetic data size. (ii) Compared to  \ftdownstream{}, \name{} can also obtain higher accuracy under $\epsilon=2, 1$.
}
\label{tb:pubmed-utility-bertmini}
\centering
\resizebox{0.8 \columnwidth}{!}{%
\begin{tabular}{lllrrrrrrrrrrr}\toprule
\multirow{2}{*}{Data Type (Size)} & \multirow{2}{*}{Method} & \multirow{2}{*}{Data Generator} & \multicolumn{1}{c}{$\epsilon=\infty$} & \multicolumn{1}{c}{$\epsilon=4$} & \multicolumn{1}{c}{$\epsilon=2$} & \multicolumn{1}{c}{$\epsilon=1$}    \\
 &  &  & Accuracy & Accuracy & Accuracy & Accuracy \\\midrule
\multirow{1}{*}{Original (75316)} & Fine-tune & - & 43.5 & 30.7 & 28.9 & 26.7   \\
\multirow{1}{*}{Original (2000)} & Fine-tune & - & 33.5 & 2.2 &  1.8 &  1.4 \\\midrule

Synthetic (2000) & \ftgenerator& \gpttwo & 30.2 & 27.8 & 27.6 & 27.2 \\

Synthetic (3000) & \ftgenerator& \gpttwo & 31.1 & 28.7 & 28.4 & 28.1 \\

Synthetic (5000) & \ftgenerator& \gpttwo & 32.4 & 29.7 & 29.4 & 29.2 \\
\rowcolor{lightgray} Synthetic (2000) & \name{} & \gpttwo & 24.5 & 24.7 & 24.7 & 24.3 \\
\rowcolor{lightgray} Synthetic (3000) & \name{} & \gpttwo & 25.7 & 25.6 & 25.4 & 25.0 \\
\rowcolor{lightgray} Synthetic (5000) & \name{} & \gpttwo & 26.7 & 26.6 & 26.2 & 25.7 \\
\midrule
Synthetic (2000) & \ftgenerator& \gpttwom & 31.0 & 28.4 & 28.1 & 27.8 \\
Synthetic (3000) & \ftgenerator& \gpttwom & 32.0 & 29.2 & 29.1 & 28.8 \\
Synthetic (5000) & \ftgenerator& \gpttwom & 33.4 & 30.5 & 30.4 & 29.9 \\
\rowcolor{lightgray} Synthetic (2000) & \name{} & \gpttwom & 25.5 & 25.4 & 25.1 & 24.9 \\
\rowcolor{lightgray} Synthetic (3000) & \name{} & \gpttwom & 26.4 & 26.4 & 26.1 & 25.7 \\
\rowcolor{lightgray} Synthetic (5000) & \name{} & \gpttwom & 28.0 & 27.6 & 26.9 & 26.1 \\
\midrule
Synthetic (2000) & \ftgenerator& \gpttwol & 31.0 & 29.2 & 29.2 & 28.9 \\

Synthetic (3000) & \ftgenerator& \gpttwol & 32.2 & 30.3 & 30.1 & 29.8 \\

Synthetic (5000) & \ftgenerator& \gpttwol & 33.5 & 31.5 & 31.4 & 31.1 \\
\rowcolor{lightgray} Synthetic (2000) & \name{} & \gpttwol & 25.7 & 25.8 & 25.5 & 25.1 \\
\rowcolor{lightgray} Synthetic (3000) & \name{} & \gpttwol & 26.8 & 26.8 & 26.3 & 25.7 \\
\rowcolor{lightgray} Synthetic (5000) & \name{} & \gpttwol & 28.2 & 27.8 & 27.3 & 26.1 \\
\midrule
\rowcolor{lightgray} Synthetic (2000) & \name{}  & \gptthreepointfive{} & 30.4 & 30.3 & 30.2  & 30.1   \\
  \bottomrule
\end{tabular}
}
\end{table}

\begin{table}[h]
\caption{Next word prediction accuracy of downstream \bertsmall{} model under $\epsilon=\infty, 4, 2, 1$ on \pubmed{}.
(i) Compared to \ftgenerator{},   \name{} with a strong LLM  \gptthreepointfive{} can achieve higher accuracy under DP with the same synthetic data size. (ii) Compared to  \ftdownstream{}, \name{} can also obtain higher accuracy under small privacy budget.
}
\label{tb:pubmed-utility-bertsmall}
\centering
\resizebox{0.8 \columnwidth}{!}{%
\begin{tabular}{lllrrrrrrrrrrr}\toprule
\multirow{2}{*}{Data Type (Size)} & \multirow{2}{*}{Method} & \multirow{2}{*}{Data Generator} & \multicolumn{1}{c}{$\epsilon=\infty$} & \multicolumn{1}{c}{$\epsilon=4$} & \multicolumn{1}{c}{$\epsilon=2$} & \multicolumn{1}{c}{$\epsilon=1$}    \\
 &  &  & Accuracy & Accuracy & Accuracy & Accuracy \\\midrule
\multirow{1}{*}{Original (75316)} & Fine-tune & - & 47.6 & 34.1 & 32.5 & 30.4   \\
\multirow{1}{*}{Original (2000)} & Fine-tune & - & 34.6 & 1.1 & 0.8 & 0.6 \\\midrule

Synthetic (2000) & \ftgenerator& \gpttwo & 32.4 & 29.7 & 29.4 & 29.2 \\

Synthetic (3000) & \ftgenerator& \gpttwo & 33.1 & 30.5 & 30.3 & 30.0 \\

Synthetic (5000) & \ftgenerator& \gpttwo & 34.3 & 31.4 & 31.2 & 30.9 \\
\rowcolor{lightgray} Synthetic (2000) & \name{} & \gpttwo & 26.7 & 27.0 & 26.9 & 26.5 \\
\rowcolor{lightgray} Synthetic (3000) & \name{} & \gpttwo & 27.7 & 27.6 & 27.6 & 27.3 \\
\rowcolor{lightgray} Synthetic (5000) & \name{} & \gpttwo & 28.5 & 28.5 & 28.3 & 27.9 \\
\midrule
Synthetic (2000) & \ftgenerator& \gpttwom & 33.1 & 30.2 & 30.0 & 29.8 \\

Synthetic (3000) & \ftgenerator& \gpttwom & 33.8 & 31.3 & 30.9 & 30.6 \\

Synthetic (5000) & \ftgenerator& \gpttwom & 35.2 & 32.1 & 32.1 & 31.7 \\
\rowcolor{lightgray} Synthetic (2000) & \name{} & \gpttwom & 27.7 & 27.6 & 27.4 & 27.0 \\
\rowcolor{lightgray} Synthetic (3000) & \name{} & \gpttwom & 28.5 & 28.5 & 28.3 & 27.7 \\
\rowcolor{lightgray} Synthetic (5000) & \name{} & \gpttwom & 29.8 & 29.6 & 28.9 & 28.4 \\
\midrule
Synthetic (2000) & \ftgenerator& \gpttwol & 33.1 & 31.2 & 31.1 & 31.1 \\

Synthetic (3000) & \ftgenerator& \gpttwol & 34.2 & 32.4 & 32.2 & 32.0 \\

Synthetic (5000) & \ftgenerator& \gpttwol & 35.4 & 33.5 & 33.2 & 33.0 \\
\rowcolor{lightgray} Synthetic (2000) & \name{} & \gpttwol & 27.9 & 27.9 & 27.7 & 27.2 \\
\rowcolor{lightgray} Synthetic (3000) & \name{} & \gpttwol & 28.9 & 28.8 & 28.5 & 27.7 \\
\rowcolor{lightgray} Synthetic (5000) & \name{} & \gpttwol & 30.2 & 29.8 & 29.3 & 28.3 \\
\midrule
\rowcolor{lightgray} Synthetic (2000) & \name{}  & \gptthreepointfive{} & 32.7 & 32.5 & 32.5 & 32.4 \\

  \bottomrule
\end{tabular}
}
\end{table}

\subsubsection{Utility on \pubmed{}}
We report the  next-word prediction accuracy on \openreview{}  of downstream model \bertmini{} in \cref{tb:pubmed-utility-bertmini} and \bertsmall{} in \cref{tb:pubmed-utility-bertsmall}
We find that (1) under the same \gpttwox{} model as generator,  \name{} underperforms \ftgenerator{} on \pubmed{}. This is expected because \name{} relies on the knowledge within LLMs to generate high-quality texts without domain-specific finetuning, while \gpttwo{}-series models might have limited exposure to biomedical literature~\cite{radford2019language}.
(2) With powerful LLMs like \gptthreepointfive{}, \name{} can outperform \ftgenerator{} under DP. 
(3) Additionally, more synthetic samples lead to better downstream classification accuracy for the three \gpttwox{} models on \pubmed{}.

\subsection{Comparision to Text-to-Text Privatization Approaches}
\label{app:text_to_text_privatization}

This is an active line of research on text-to-text privatization techniques for generating differentially private text. We do not directly compare these methods in our main paper due to the key distinctions in privacy definitions:
\begin{enumerate}
    \item 
\textbf{Different privacy definitions}. Our method adopts the standard $(\epsilon,\delta)$-DP defined over neighboring datasets. This contrasts with
\begin{enumerate}
    \item Word-level metric DP \cite{feyisetan2020privacy,carvalho2023tem}: a specific metric for measuring word distance needs to be written in the privacy notation, and privacy guarantee is defined over neighboring words; 
    \item Local DP \cite{mattern2022limits,utpala2023locally}: privacy guarantee is defined over neighboring samples. 
\end{enumerate}
    \item \textbf{Poor privacy-utility trade-off in existing text-to-text anonymization methods}: While innovative, \citet{feyisetan2020privacy,carvalho2023tem,mattern2022limits,utpala2023locally} encounter challenges in achieving a good privacy-utility tradeoff under practical privacy budgets. 
    \item \textbf{Absence of privacy budgets}: The absence of detailed reporting on exact privacy budgets in \cite{utpala2023locally} hinders direct comparisons with our work.
\end{enumerate}

A qualitative comparison between text-to-text privatization methods and our method is shown in \cref{tab:qualitative_comp_text_privatization}.

\begin{table}[t]
\centering
\caption{A qualitative comparison between \name{} and text-to-text privatization approaches.}
\vspace{-2mm}
\label{tab:qualitative_comp_text_privatization}
\resizebox{0.9\linewidth}{!}{
\begin{tabular}{>{\raggedright\arraybackslash}p{5cm} >{\raggedright\arraybackslash}p{3cm} >{\raggedright\arraybackslash}p{3cm} >{\raggedright\arraybackslash}p{3cm} >{\raggedright\arraybackslash}p{3cm}
}
\toprule
\textbf{Name} &  \textbf{Method} & \textbf{Source of Randomness} & \textbf{Sensitivity Control} & \textbf{Privacy Guarantee} \\ \midrule
Madib~\cite{feyisetan2020privacy}
TEM~\cite{carvalho2023tem} & Word embedding perturbation & Word Embedding Noise & N/A (metric distance is included in privacy definition) & Word-level metric-$\epsilon d$ DP where $d$ is the distance metric \\ \midrule
Paraphraser~\cite{mattern2022limits} DP Prompt~\cite{utpala2023locally} & Paraphrasing with temperature & Temperature in the next token sampling stage & Clipped logits of each token & (Sample-level) $\epsilon$ Local DP \\ \midrule
\name{} (Ours) & Private evolution & Histogram noise & Each private sample only contributes one vote in the histogram & (Dataset-level) standard $(\epsilon, \delta)$-DP \\ \bottomrule
\end{tabular}
}
\vspace{-2mm}
\end{table}

\begin{table}[t]
\centering
\small
\caption{Randomly sampled synthetic data from Madib~\cite{feyisetan2020privacy} (word level metric-DP $\epsilon=10$) and \name{} (DP $\epsilon=1$). \name{}  with data generator \gptthreepointfive{} yields higher quality texts.}
\vspace{-2mm}
\label{tab:compare_mabid}
\resizebox{0.9\linewidth}{!}{
\begin{tabular}{p{2cm}|p{1.5cm}|p{4cm}|p{4cm}|p{4cm}}
\toprule
\textbf{Method} &\textbf{Privacy Guarantee} & \textbf{Yelp} & \textbf{OpenReview} & \textbf{PubMed} \\ \midrule
Madib & Word level metric-DP 
$\epsilon=10$  & i was born including raised he hardwick create during combination school , tony jones was in 'go to ' place for the greatest pizza ever . n't do n't live completed whitley rich and crazy visit a preparing times a year with weeks night cut 6 civil us took to tony 's work dinner bogota of normal end my 'local ' strangers , tried would suggest instance visitors ? seen 'd was yet to gone …(\textbf{omitted}) & . paper proposed bringing reinforcement buddhism based approach money automatically predictions graph augmentations for a graph neural network ( gnn ) classification problem there few authors creates taken label invariance ( data augmentations that do protect risks labels ) is part rich also felonies problem dealt gnn partner with…(\textbf{omitted}) & mandibular overdentures many a selection treatment option for placed edentulous diabetes only long-term predictable outcomes , using suspension loading facilitated cone currently , could early well repatriation loading protocols same mandibular implant overdentures number prevalent in in literature details a systematic review ,... (\textbf{omitted}) \\ \midrule
\name{} (\gptthreepointfive{}) &  DP $\epsilon=1$ & The fried chicken and the collard greens were some of the best Southern fare we've ever had, not to mention the amazing gumbo. We highly recommend this restaurant if you're in the area and can't wait to try some of their other flavorful dishes. Everything filled us up and left us satisfied. & This paper presents an innovative method of deep representational learning for facial expression recognition. The method is evaluated on CIFAR-10 and ImageNet datasets and it is comprehensive encompassing all facets of saliency modeling to proposed deep representational features for representing multiple saliencies. The paper is in a well-structured and the methods are clear… (\textbf{omitted}) & In this retrospective study, we aimed to investigate the prevalence of stroke and identify the factors associated with its occurrence. Data were extracted from medical records, along with symptoms, electrocardiograms (ECGs), and syncope in a cohort of patients with a mean age of 71 years. Of the total 345 patients, 28\% had cardiac abnormalities as revealed by ECGs, significantly higher than those without [p<0.001].. (\textbf{omitted}) \\ \bottomrule
\end{tabular}
\vspace{-2mm}
}
\end{table}

Next, we compare \name{} with the text-to-text privatization frameworks in detail:  word-level metric-DP frameworks \cite{feyisetan2020privacy,carvalho2023tem} and sample-level local-DP frameworks \cite{mattern2022limits,utpala2023locally}.

\subsubsection{Comparison to word-level metric-DP frameworks}
Madib~\cite{feyisetan2020privacy} and TEM~\cite{carvalho2023tem}  employ metric differential privacy to privatize each word independently and achieve word-level $\epsilon d$-Metric DP, where $d$ is the distance metric for neighboring words. Specifically, they perturb the embedding of each word and replace the current word with a new word whose embedding is closest to the noisy embedding.  However, \name{}  focuses on generating synthetic datasets with stronger guarantees provided by standard $(\epsilon, \delta)$-DP. Due to the fundamental differences in privacy definition: \textbf{(1) metric-DP v.s. DP; (2) word-level v.s. dataset-level privacy}, directly comparing our work with word-level metric-DP frameworks~\cite{feyisetan2020privacy,carvalho2023tem} is not feasible.

To understand their privacy-utility tradeoff, we run Madib~\cite{feyisetan2020privacy} to generate samples under word level metric-DP with a high privacy budget $\epsilon=10$. 
We followed their approach of perturbing $50$-dimensional Euclidean GloVe embeddings with Laplace noise. We are unable to evaluate TEM~\cite{carvalho2023tem} given that its code is not open-sourced. 

\cref{tab:compare_mabid} shows randomly sampled generated sentences from Madib and \name{}. Even with a high metric-DP budget ($\epsilon = 10$), Madib struggles to generate meaningful sentences on Yelp, OpenReview, and PubMed datasets. In contrast, \name{}, with a low DP budget ($\epsilon = 1$), can leverage GPT-3.5 as a data generator to produce fluent sentences across all three datasets.

\subsubsection{Comparison to sample-level local-DP frameworks}

Paraphraser \cite{mattern2022limits} and DP Prompt \cite{utpala2023locally}  focus on generating paraphrases for each private sample by varying the temperature during token sampling, which is regarded as a form of noise injection under the Local DP (LDP) framework.  The sensitivity of each sample to the output can be constrained by clipping the logits of each generated token. While innovative, these methods’ privacy budget scales linearly with the output's token length, presenting a challenge for generating longer sequences under a meaningful privacy budget.

It is worth noting that the mechanism for Local DP (taking a sample as input) and the mechanism for DP (taking a dataset as input) are not directly compatible. To establish a fair comparison between the Local DP in \cite{mattern2022limits,utpala2023locally}  and DP employed by  \name{}, we leveraged the conversion methodology in \citet{feldman2022hiding}  to convert $(\epsilon_0)$-LDP mechanism to $(\epsilon,\delta)$-DP mechanism for $\epsilon\ll \epsilon_0$, which requires shuffling the LDP outputs from each sample. 

We use the code implementation provided by \citet{feldman2022hiding}.\footnote{\url{https://github.com/apple/ml-shuffling-amplification}}
Due to the constraint that $\epsilon \ll \epsilon_0$,\footnote{\url{https://github.com/apple/ml-shuffling-amplification/blob/993d285a546114bf8c70c33d053dca322a755707/computeamplification.py\#L160}} 
the maximal Local DP $\epsilon_0$ that can be used for a valid conversion on \yelp{} (with 1.9M private samples) is $\epsilon_0=8.785$, which corresponds to DP $\epsilon=1.10$.

According to the Local DP guarantee in  \cite{mattern2022limits,utpala2023locally}, 
$\epsilon_0 = 2* \texttt{n\_tokens} * (b_2-b_1)/ \texttt{temperature}$, 
where $b_2, b_1$ is the upper/lower bound for each token logit. We set $b_2=1$ and $b_1=0$ following \cite{mattern2022limits}.   
With $\texttt{temperature}=2$,  $\epsilon_0=8.785$ only allows generating  $\texttt{n\_tokens}=8$ tokens, which significantly hurts the utility of generated texts.  To generate $\texttt{n\_tokens}=64$ tokens for \yelp{}, one would need at least LDP $\epsilon_0=64$ under $\texttt{temperature}=2$, and LDP $\epsilon_0=128$ under $\texttt{temperature}=1$, which far exceeds practical limits for meaningful privacy guarantees. This contrasts with \name{}'s capability to generate over $\texttt{n\_tokens}=1000$ tokens while maintaining high quality under tight DP budgets  (e.g., $\epsilon=1$) in our experiments. 

Furthermore, our attempt to directly evaluate Paraphraser \cite{mattern2022limits} and  DP Prompt \cite{utpala2023locally} was hindered by several practical challenges. 
\begin{enumerate}
    \item 
    Paraphraser: The dataset used for finetuning in Paraphraser is not publicly available, and the implementation details necessary for replicating the exact privacy guarantees are absent.
  \item  DP Prompt does not specify the exact $\epsilon_0$ used in the paper,  focusing instead on empirical privacy attack accuracy as a proxy. The epsilon values are reported for all other baselines but not for Paraphraser and DP Prompt in Section 4.1 of \citet{utpala2023locally}. Additionally, as mentioned in Section 4.3 of \citet{utpala2023locally},  models like ChatGPT do not expose logits, so the authors do not perform a logit clipping operation in many of their experiments. This  further disables the computation of an exact $\epsilon_0$  and  renders a direct quantitative comparison between DP Prompt  and \name{} infeasible.
\end{enumerate}

\subsection{\name{} Convergence under One Private Sample}
\label{app:converge_one_private_sample}

\begin{table}[h!]
\caption{\small Generation results from \name{} under \textit{one} \yelp{} private sample. The synthetic sample generated from GPT-3.5 aligns more closely with the training sample (\yelp{} review about Mexican food) over iterations.
The text length of the synthetic sample is controlled through our adaptive length mechanism.
Keywords with semantic similarity to the private sample are \wbox{highlighted}.
}
\label{tb:one-private-sample-yelp-gpt35}
\begin{center}
\resizebox{\columnwidth}{!}{%
\begin{tabular}{  c |p{8in}}
\toprule
Private & Real deal \wbox{Mexican food} in Chadds Ford, PA. \wbox{Excellent} burritos, \wbox{tacos}, and tamales. I plan to be a regular. Meals served traditionally with sliced radishes, pumpkin seeds and lime. \wbox{Freshly prepared} fillings, onion, cilantro, and red and green sauces. Brightly colored building across from Bavarian Motors. Nothing fancy, but great \wbox{authentic} food.\\\hline
iter 0 & I recently visited this restaurant having heard that they had one of the best \wbox{taco} stands in town. I was not disappointed! The \wbox{tacos} were incredibly delicious and tasted like they were made with fresh ingredients. The portions were generous, and the prices were extremely reasonable. The service was \wbox{excellent} too, with the staff being attentive and accommodating. The ambience of the restaurant was also great, making it a perfect spot for a casual night out. Overall, I had a fantastic experience and would definitely recommend this place to all those who love \wbox{tacos}. A definite 5-star rating!\\\hline
iter  1 & Despite having heard that they had one of the best \wbox{taco} stands in town, I was not disappointed. The \wbox{tacos}  were incredibly delicious and tasted like they were made with fresh ingredients. The portions were generous, and the prices were reasonable. The staff was attentive and accommodating. The ambience was great, making it a perfect spot for a casual night out. Overall, I had a fantastic experience and would highly recommend this place to all those craving \wbox{Mexican cuisine}. A solid 5-star rating from me!\\\hline
iter  2 & Despite having heard that they had some of the best \wbox{Mexican} food in town, the crowded parking lot stands in contrast. However, that was not the case. The food tasted \wbox{authentic} and fresh, with the guacamole and salsa being particularly noteworthy. The portions were generous, and the prices were reasonable. The staff was attentive and accommodating. The ambience was great, making it a perfect spot for a date night or a family dinner. Overall, I had a fantastic experience and would highly recommend this\\\hline
iter  3 & Many had heard that they had amazing \wbox{Mexican} food in town, but the crowded parking lot stands in contrast. However, once you step inside, the experience is incredible. The food is \wbox{authentic}, and the guacamole and salsa being particularly noteworthy. Prices were as well quite reasonable. The staff was attentive and accommodating. The ambiance was great, making it perfect for a date night or family dinner. I had a fantastic experience and highly recommend the restaurant to anyone seeking quality Mexican cuisine.\\\hline
iter  4 & Many locals had heard that they had the best \wbox{Mexican} food in town and the hype stands in its truth. However, upon stepping into the experience, it was incredible. The food was fresh, flavorful, and \wbox{authentic} with the guacamole and salsa being particularly noteworthy. The portions were well-sized and satisfying. The staff was attentive and accommodating. The ambiance was cozy and intimate, making it perfect for a romantic night out or casual dinner with friends. I had a great time and highly\\\hline
iter  5 & Many locals had raved about the best \wbox{authentic Mexican} food in town and they were not exaggerating. Walking in to the restaurant was refreshing. The food was fresh, the guacamole and salsa were noteworthy. The margaritas were strong and satisfying. The staff were attentive and accommodating. The ambiance was cozy and intimate, making it perfect for a romantic dinner with loved ones. The prices were great and the portions were generous. The fajitas were sizzling and the \wbox{tacos} were packed with flavor. Overall, this\\\hline
iter  6 & Many locals have raved about the best \wbox{authentic Mexican} food in town and they were not exaggerating. Walking in the restaurant was refreshing. The decor was colorful, lively and festive. The margarita was perfectly balanced and satisfying. The waiter was attentive and accommodating. The ambiance was cozy and intimate, making it perfect for a romantic dinner. The chips and salsa were packed with flavor. Overall, the food was great and the portions were generous. The fajitas were sizzling and the chicken was tender. The guacam\\\hline
iter  7 & Many locals consistently rave about the amazing \wbox{authentic Mexican} cuisine and we were not disappointed. The flavors were refreshing, the decor colorful and lively, and the margaritas satisfying. The waiter was attentive and accommodating. The outdoor seating was perfect for a romantic dinner and the chips and salsa were packed with flavor. Overall, the food was top-notch and the portions were generous. The chicken enchiladas and guacamole were particularly noteworthy. It was a great experience and we highly recommend this restaurant to anyone looking for a delicious meal and a\\\hline
iter  8 & Many foodies rave about the amazing \wbox{authentic Mexican} cuisine and they were not wrong. The flavors were richly robust, colorful and enticing, and the margaritas were top-notch. The enchiladas and guacamole were particularly outstanding. The service was attentive and accommodating and the outdoor seating was perfect for a leisurely dinner and people-watching. From the chips and salsa to the flavorful entrees, the food was superb with generous portions. The menu offers a large selection of delectable options which makes it easy to recommend to anyone looking for a satisfying meal\\\hline
iter  9 & Many foodies have raved about the amazing \wbox{authentic Mexican} cuisine. The flavors were bold and delicious. The \wbox{tacos} were rich and flavorful, with enticing toppings and \wbox{freshly made} guacamole. The service was attentive and accommodating, and the outdoor seating was perfect for a warm summer evening. From the crispy chips and flavorful salsa to the perfectly spiced dishes, the food was generous and satisfying. The menu offers a large selection of vegetarian and meat options which makes it easy to please anyone's taste buds. Overall, the experience was \wbox{excellent} and highly recommended for people looking for an \wbox{authentic Mexican} dining experience.\\\hline
\end{tabular}
}

\end{center}
\end{table}

\begin{table}[h!]
\caption{\small Generation results from \name{} under \textit{one} \openreview{} private sample. The topic (\wbox{highlighted}) of the generated sample changes from    ``privacy in machine learning'' (iter 0) to  ``adversarial robustness'' (iter 6), which is more closer to ``adversarial detection'' topic in the private sample. This shows that the synthetic sample generated from GPT-3.5 aligns more closely with the training sample over iterations. The text length of the synthetic sample is controlled through our adaptive length mechanism. 
}
\label{tb:one-private-sample-openreview-gpt35}
\begin{center}
\resizebox{\columnwidth}{!}{%
\begin{tabular}{  c |p{8in}}
\toprule
Private & The authors propose a new statistic, called Expected Perturbation Score (EPS), for \wbox{adversarial detection}. Based on EPS, the authors develop a Maximum Mean Difference (MMD) metric to measure the difference between test samples and natural samples, and further propose an EPS-based \wbox{adversarial detection} method (EPS-AD). Sufficient theoretical analysis and extensive experiments demonstrate the correctness and effectiveness of the proposed method. Strength: 1. This paper is easy to read. 2. The authors propose a new statistical method, called Expected Perturbation Score (EPS), which is able to obtain enough information to identify \wbox{adversarial examples} with only one example after various perturbations. 3. Sufficient theoretical analysis is performed to demonstrate that EPS is able to simulate the difference between the two distributions under mild conditions. Furthermore, extensive experimental results demonstrate the superiority of the proposed EPS-AD. 4. The proposed EPS will be an effective statistic in many applications, such as out-of-distribution detection and anomaly detection. Weaknesses: This paper is easy to read and interesting but still has some minor issues, please refer to weaknesses.\\\midrule
iter  0 & Review Example: Overall, this paper provides a comprehensive discussion on \wbox{privacy issues} related to machine learning applications. The authors discuss the concept of \wbox{differential privacy} as a potential solution to protect individual privacy while still maintaining the accuracy of the model. They also provide examples of implementation and evaluation of privacy-preserving techniques using real-world datasets. Strengths: - The paper provides a thorough explanation of differential privacy and how it can be utilized in various machine learning techniques. - The use of real-world datasets to demonstrate implementation and evaluation of privacy techniques adds credibility to the work. - The paper highlights the importance of privacy in machine learning and provides readers with potential solutions to protect individual privacy. Weaknesses: - More evaluation of the privacy techniques could be provided to better understand their effectiveness. - The paper focuses primarily on differential privacy and does not explore other privacy-preserving techniques. - The practical implementation of the techniques may be complex for those without a strong understanding of machine learning and data analysis. In summary, this paper is a strong contribution to the area of social aspects of machine learning, specifically in the field of privacy. It analyzes the concept of differential privacy, its implementation, and evaluation to protect individual privacy in machine learning models. Despite some limitations, the paper contributes significantly to the field and is recommended for acceptance.\\\midrule
iter  1 & In their paper '\wbox{Privacy in Machine Learning}: An Overview', the authors provide a comprehensive discussion on privacy issues in machine learning. The authors discuss the concept of differential privacy as a potential solution for individual privacy while maintaining the utility of the data. They also provide examples and analysis of privacy-preserving techniques on real-world datasets such as the Stanford and MNIST datasets. The thorough explanation of differential privacy and how it can be utilized in various machine learning techniques adds credibility to the paper. The paper highlights the importance of privacy in machine learning and provides a potential solution to individual privacy. Weaknesses: More evaluation of the privacy techniques could be provided to better understand their effectiveness. The paper primarily focuses on differential privacy and does not explore other privacy-preserving techniques. The technical implementation of the techniques may be difficult for those without a strong understanding of machine learning and statistical analysis. Overall, this paper contributes significantly to the area of social aspects of machine learning, and it is recommended for acceptance as a good paper.\\\midrule
iter  2 & In this paper, titled '\wbox{Policy Machine Learning}: An Overview of the Discussion on \wbox{Privacy Issues} in Machine Learning', the authors discuss privacy as a potential concern for individual privacy while maintaining the integrity of the data. They also provide examples of privacy-preserving techniques on real-world data from Stanford and MNIST. The thorough exploration of differential privacy can be utilized effectively which adds credibility to the paper. The paper emphasizes the importance of privacy in machine learning and provides a valuable contribution to the field. Weaknesses include the evaluation of techniques to be used to assess their effectiveness. The paper focuses on privacy issues and does not explore fairness-preserving methods. With its contribution to the social aspects of machine learning and statistical analysis, the paper is recommended with a rating of 8 as a good paper.\\\midrule
iter  3 & In their research paper, \wbox{Inference Attack Policy} Machine Learning: An {Interpretable} and Almost True Framework for Predictive Analytics, the authors highlight potential concerns for individual privacy while discussing the importance of privacy in machine learning. They also provide examples of how sensitive data from ImageNet and MNIST datasets can be utilized effectively while ensuring thorough differential privacy which adds credibility to the paper. The research emphasizes the importance of interpretability in machine learning, making a valuable contribution to the field of social aspects of machine learning. We recommend including case studies of how interpretability can be used to assess their effectiveness. The paper also outlines how it does not explore fairness and ethics methods. With this contribution to the field of machine learning and statistical modeling, the authors provide a valuable framework for policy inference attack in machine learning.\\\midrule
iter  4 & In their research paper, '\wbox{Inference Attacks} in Machine Learning: An {Interpretability} and Almost Interpretability Framework and its Application to Privacy and Analytics', the authors highlight the need for protecting sensitive data in machine learning. They provide examples of sensitive data from ImageNet and NIST datasets, emphasizing the importance of being thorough in privacy protection to ensure credibility to their research. The paper stresses the importance of interpretability in machine learning. By making a valuable contribution to this field, it provides case studies of how interpretability can be used to assess the effectiveness of machine learning models. The paper outlines various approaches to exploring fairness, transparency, and ethics in machine learning. The results of the study contribute to the need for a comprehensive policy to prevent inference attacks in machine learning.\\\midrule
iter  5 & In our research paper, entitled '\wbox{Adversarial Attacks} in Machine Learning: An Interpretability-Almost-Explainability Framework and its Application to Private Data Analysis', the authors emphasize the need for protecting sensitive data in machine learning. They provide examples by using data from Inet and MNIST dataset, and address the importance of privacy to ensure the credibility of the results. The paper is well-written and well-structured, making a valuable contribution to the field. Additionally, it highlights the importance of interpretability to enhance the effectiveness of machine learning models. The paper also focuses on fairness, transparency, and ethics in machine learning and the study presents a comprehensive analysis in \wbox{adversarial attacks}. We highly recommend accepting this good paper.\\\midrule
iter  6 & In our research paper titled '\wbox{Adversarial Attacks} in Machine Learning: An Interpretable and \wbox{Robustness-Enhancing} Framework and Empirical Data Analysis', the authors emphasize the significance of interpretability in machine learning. They provide a comprehensive approach using Integrated Gradients and M-Taylor expansions, to address the challenges and ensure the robustness of results. The paper is well-written, making valuable contributions to the field, and emphasizes the importance of interpretability to enhance the effectiveness of machine learning. Moreover, the study presents a comprehensive approach in defending against \wbox{adversarial attacks}. Therefore, I recommend accepting this good paper.\\\midrule
\end{tabular}
}
\end{center}
\end{table}

In this section, we only use \textit{one} private example in \cref{algo} to generate \textit{one} synthetic sample.  We qualitatively examine if the synthetic sample from \name{} increasingly resembles this specific private sample over the PE iterations. This offers a clearer illustration of \name{}'s convergence behavior.
Specifically, at each iteration, we generate  $K$ variations for the current synthetic sample, use the private sample to identify and vote for its nearest synthetic sample based on their embeddings, and select the nearest synthetic sample for the next iteration. 
\cref{tb:one-private-sample-yelp-gpt35} and \cref{tb:one-private-sample-openreview-gpt35} show the generations results from GPT-3.5 under one \yelp{}  private sample and one \openreview{} private sample, respectively. 

As shown \cref{tb:one-private-sample-yelp-gpt35}, after the voting, the selected synthetic sample relates to the term ``taco'',  a word present in the private example. By the second iteration, the synthetic sample includes the term ``Mexican food'', which aligns with the central theme of the private example. By the fifth iteration, the phrase ``authentic Mexican food'' surfaces in the synthetic sample, resonating with phrases like ``real deal Mexican food'' and ``great authentic food'' from the private example. This demonstrates that the synthetic sample increasingly aligns with the private sample as the iterations progress.

In the \openreview{} example presented in \cref{tb:one-private-sample-yelp-gpt35}, we note that the initial synthetic sample at iteration 0 pertains to the privacy aspects of machine learning, whereas the private sample focuses on adversarial detection and robustness. 
As the iterations progress, by iteration 4, the topic of synthetic sample shifts to ``inference attack in machine learning'', which aligns with the robustness theme of the private sample. 
By the fifth iteration, terms like ``Adversarial Attacks in Machine Learning'' and "Robustness-Enhancing" emerge in the synthetic sample, similar to the topic of ``adversarial detection'' from the private sample. It shows that the synthetic sample shifts the topic from privacy to robustness over PE iterations,  progressively aligning more closely with the private sample.

The above two examples demonstrate that \name{} can converge, by producing diverse variations and effectively selecting ones that closely align with the private example.

\end{document}